\documentclass{article}

\usepackage{microtype}
\usepackage{graphicx}
\usepackage{subfigure}
\usepackage{booktabs} % for professional tables

\usepackage[accepted]{icml2020}

\usepackage{amsmath,amssymb,amsfonts,amsthm}

\usepackage{mathtools}  
\usepackage{tabulary}
\usepackage{booktabs}

\usepackage{graphics}
\usepackage{ textcomp }

\usepackage{pdfpages}
 \usepackage{colortbl}
 \usepackage[english]{babel}
 \usepackage[utf8]{inputenc}
 \usepackage{lipsum}
\usepackage{ nicefrac}
\usepackage{bbm} 
\usepackage{comment}

\usepackage{color}

\usepackage[hidelinks]{hyperref}
\definecolor{darkred}{RGB}{150,0,0}
\definecolor{darkgreen}{RGB}{0,150,0}
\definecolor{darkblue}{RGB}{0,0,150}
\hypersetup{colorlinks=true, linkcolor=red, citecolor=blue, urlcolor=darkblue}

\newcommand{\Otilde}{\widetilde{\mathcal{O}}}

\newtheorem{assumption}{Assumption}
\newtheorem{theorem}{Theorem}[section]

\newtheorem{lemma}[theorem]{Lemma}
\newtheorem{Proposition}[theorem]{Proposition}

\theoremstyle{definition}
\newtheorem{definition}{Definition}[section]

\newcommand{\norm}[1]{\left\lVert#1\right\rVert}

\icmltitlerunning{Safe Linear Thompson Sampling with Side Information}

\begin{document}

\onecolumn
\icmltitle{Safe Linear Thompson Sampling with Side Information}



\icmlsetsymbol{equal}{}

\begin{icmlauthorlist}
\icmlauthor{Ahmadreza Moradipari}{}
\icmlauthor{Sanae Amani}{}
\icmlauthor{Mahnoosh Alizadeh}{}
\icmlauthor{Christos Thrampoulidis}{}
\end{icmlauthorlist}

%\icmlaffiliation{to}{Department of Computation, University of Torontoland, Torontoland, Canada}
%\icmlaffiliation{goo}{Googol ShallowMind, New London, Michigan, USA}
%\icmlaffiliation{ed}{School of Computation, University of Edenborrow, Edenborrow, United Kingdom}

\icmlcorrespondingauthor{Ahmadreza Moradipari}{ahmadreza_moradipari@ucsb.edu}
%\icmlcorrespondingauthor{Eee Pppp}{ep@eden.co.uk}

\icmlkeywords{Machine Learning, ICML}

\vskip 0.3in



%\printAffiliationsAndNotice{} % otherwise use the standard text.

\begin{abstract}
The design and performance analysis of bandit algorithms in the presence of stage-wise safety or reliability constraints has recently garnered significant interest. In this work, we consider the linear stochastic bandit problem under additional \textit{linear safety constraints} that need to be satisfied at each round. We provide a new safe algorithm based on linear Thompson Sampling (TS) for this problem and show a frequentist regret of order $\mathcal{O} (d^{3/2}\log^{1/2}d \cdot T^{1/2}\log^{3/2}T)$, which remarkably matches the results provided by \cite{abeille2017linear} for the standard linear TS algorithm in the absence of safety constraints. We compare the performance of our algorithm with UCB-based safe algorithms and highlight how the inherently randomized nature of TS leads to a superior performance in expanding the set of safe actions the algorithm has access to at each round.
\end{abstract}

\section{Introduction}
\label{introduction}

The application of stochastic bandit optimization algorithms to safety-critical systems has received significant attention in the past few years. In such cases, the learner repeatedly interacts with a system with uncertain reward function and operational constraints. In spite of this uncertainty, the learner needs to ensure that her actions do not violate the operational constraints {\it at any round of the learning process}. As such, especially in the earlier rounds, there is a need to choose actions with caution, while at the same time making sure that the chosen action provide sufficient learning opportunities about the set of safe actions.  Notably, the actions deemed safe by the algorithm might not originally include the optimal action.  This uncertainty about safety and the resulting conservative behavior means the learner could experience additional regret in such constrained environments.

In this paper, we focus on  a special class of stochastic bandit optimization problems where the reward is a linear function of the actions. This class of problems, referred to as linear stochastic bandits (LB), generalizes multi-armed bandit (MAB) problems to the setting where each action is associated with a feature vector $x$, and the expected reward of playing each action is equal to the inner product of its feature vector and an unknown parameter vector $\theta^{\star}$. There exists several variants of LB that study the finite \cite{auer2002finite} or infinite \cite{Dani08stochasticlinear,Tsitsiklis,abbasi2011improved} set of actions, as well as the case where the set of feature vectors can change over time \cite{pmlr-v15-chu11a,li2010contextual}.   Two efficient approaches have been developed for LB: {\it linear UCB} (LUCB) and {\it linear Thompson Sampling} (LTS).
For LUCB, \cite{abbasi2011improved} provides a regret bound of order $\mathcal{O} (d \cdot T^{1/2} \log T)$.  For LTS \cite{agrawal2013thompson,abeille2017linear} adopt a  frequentist view  and show a regret of order $\mathcal{O} (d^{3/2}\log^{1/2}d \cdot T^{1/2}\log^{3/2}T)$.
Here we provide a LTS algorithm that respects \textit{linear safety constraints} and study its performance. We formally define the problem setting before summarizing our contributions.

\subsection{Safe Stochastic Linear Bandit Model}\label{sec:setting}
\textbf{Reward function.} The learner is given a convex and compact set of actions $\mathcal{D}_0 \subset \mathbb{R}^d$. At each round $t$, playing an action $x_t \in \mathcal{D}_0$ results in observing reward \begin{align}r_t := x_t^{\top}\theta_{\star} + \xi_t,\end{align}  where $\theta_{\star} \in \mathbb{R}^d$ is a fixed, but \emph{unknown}, parameter and $\xi_t$ is a zero-mean additive noise.

\textbf{Safety constraint.} We further assume that the environment is subject to a linear constraint:  
\begin{align}
    x_t^{\top}  \mu_{\star} \leq C, \label{safty.const}
\end{align}
which needs to be satisfied by the action $x_t$ at every round $t$, to guarantee safe operation of the system.
%(referred to as the safety constraint)
Here, $C$ is a \textit{positive} constant that is \textit{known} to the learner, while  $\mu_{\star}$ is a fixed, but \emph{unknown} vector parameter. Let us denote the set of ``safe actions'' that satisfy the constraint \eqref{safty.const} as follows:
\begin{align}
    \mathcal{D}_0^s (\mu_{\star}) := \{ x \in \mathcal{D}_0 : x^{\top}  \mu_{\star} \leq C   \}. \label{action.set}
\end{align}
Clearly, $\mathcal{D}_0^s (\mu_{\star})$ is unknown to the learner, since $\mu_\star$ is itself unknown. However, we consider a setting in which, at every round $t$, the learner receives \emph{side information} about the safety set via noisy measurements:
\begin{align}\label{eq:side}
w_t = x_t^{\top}\mu_{\star}+\zeta_t,\end{align} where $\zeta_t$ is zero-mean additive noise. During the learning process, the learner needs a mechanism that allows her to use the side measurements in \eqref{eq:side} for determining the safe set $\mathcal{D}_0^s (\mu_{\star})$. 
This is critical, since it is required (at least with high-probability) that $x_t\in\mathcal{D}_0^s (\mu_{\star})$ for all rounds $t$. %Therefore 

\textbf{Regret.} The {\it cumulative pseudo-regret} of the learner up to round $T$  is defined as $R(T) = \sum_{t=1}^T  x^{\top}_{\star} \theta_{\star} - x_t^{\top} \theta_{\star}$, where $x_{\star}$ is the optimal \emph{safe} action that maximizes the expected reward over $D_0^s(\mu_\star)$, i.e., $$x_{\star} = \text{arg}\max_{x \in \mathcal{D}_0^s(\mu^*)} x^{\top} \theta_{\star}.$$

\textbf{Learning goal.} The learner's objective is to control the growth of the pseudo-regret. Moreover, we require that the chosen actions $x_t, t \in [T]$ are safe (i.e., they belong to $\mathcal{D}_0^s(\mu_\star)$ in \eqref{action.set}) with high probability. As is common, we use regret to refer to the pseudo-regret $R(T)$. 

\subsection{Contributions} 

\noindent$\bullet$~~We provide the first \emph{safe} LTS (Safe-LTS) algorithm with provable regret guarantees for the linear bandit problem with linear safety constraints.

\noindent$\bullet$~~Our regret analysis shows that Safe-LTS achieves the \emph{same} order $\mathcal{O} (d^{3/2}\log^{1/2}d \cdot T^{1/2}\log^{3/2}T)$ of   regret as the original LTS (without safety constraints) as shown by \cite{abeille2017linear}. Hence, the dependence of the regret of Safe-LTS  on the time horizon $T$ \emph{cannot} be improved modulo logarithmic factors (see lower bounds for LB in \cite{Dani08stochasticlinear,Tsitsiklis}).

\noindent$\bullet$~~We compare Safe-LTS to the existing safe versions of LUCB for  linear stochastic bandits with linear stage-wise safety constraints. We show that our algorithm has: better regret in the worst-case, fewer parameters to tune and superior empirical performance.

\subsection{ Related Work}\label{priorwork}

\textbf{Safety -} A diverse body of related works on stochastic optimization and control have considered the effect of safety constraints that need to be met during the run of the algorithm \cite{aswani2013provably, koller2018learning} and references therein. Closely related to our work, \cite{Krause,sui2018stagewise}  study \emph{nonlinear} bandit optimization with \emph{nonlinear}  safety constraints using Gaussian processes (GPs) as non-parametric models for both the reward and the constraint functions. Their algorithms  have shown great promises in robotics applications \cite{ostafew2016robust,7039601}.  Without the GP assumption, \cite{kamgar} proposes and analyzes a safe variant of the Frank-Wolfe algorithm to solve a smooth optimization problem with an unknown convex objective function and unknown \emph{linear} constraints (with side information, similar to our setting). All the above algorithms come with provable convergence guarantees, but \emph{no} regret bounds.  
To the best of our knowledge, the first work that derived an algorithm with provable regret guarantees for bandit optimization with stage-wise safety constraints, as the ones imposed on the aforementioned works, is \cite{amani2019linear}. While \cite{amani2019linear} restricts attention to a \textit{linear} setting, their results reveal that the presence of the safety constraint --even though linear-- can have a  non-trivial effect on the performance of LUCB-type algorithms. Specifically, the proposed Safe-LUCB algorithm comes with a problem-dependent regret bound that depends critically on the location of the optimal action in the safe action set --  increasingly so in problem instances for which the safety constraint is active. In \cite{amani2019linear}, the linear constraint function involves the same unknown vector (say, $\theta_{\star}$) as the one that specifies the linear reward. Instead, in Section \ref{sec:setting} we allow the constraint to depend on a new parameter vector (say, $\mu_{\star}$) to which the learner get access via side-information measurements \eqref{eq:side}. This latter setting is the direct \emph{linear} analogue to that of \cite{Krause,sui2018stagewise,kamgar} and we demonstrate that an appropriate Safe-LTS algorithm enjoys regret guarantees of the same order as the original LTS \emph{without} safety constraints. A more elaborate discussion comparing our results to \cite{amani2019linear} is provided in Section \ref{camparing with safe-Lucb}. We also mention \cite{vanroy} as another recent work on safe linear bandits. In contrast to the previously mentioned references, \cite{vanroy} defines safety as the requirement of ensuring that the \emph{cumulative} (linear) reward up to each round stays above a given percentage of the performance of a known \textit{baseline} policy. 
As a closing remark, \cite{amani2019linear,vanroy,kamgar} show that simple linear models for safety constraints might be directly relevant to several applications such as medical trials  applications, recommendation systems or managing the customers’ demand in power-grid systems. Moreover, even in more complex settings where linear models do not directly apply (e.g., \cite{ostafew2016robust,7039601}), we still believe that this simplification is an appropriate first step towards a principled study of the regret performance of safe algorithms in sequential decision settings.

\textbf{Thompson Sampling -}  Even though TS-based algorithms \cite{thompson1933likelihood}  are computationally easier to implement than UCB-based algorithms and have shown great empirical performance, they were largely ignored by the academic community until a few years ago, when a series of papers (e.g., \cite{russo, abeille2017linear, agrawal2012analysis, kaufmann2012thompson}) showed that TS achieves  optimal  performance in both frequentist and Bayesian settings. Most of the literature focused on the analysis of the Bayesian regret of TS for general settings such as linear bandits or reinforcement learning (see e.g., \cite{osband2015bootstrapped}).  More recently, \cite{russo2016information,dong2018information,dong2019performance} provided an information-theoretic analysis of TS.%, where the key tool in their approach is the {\it information ratio} which quantifies the trade-off between exploration and exploitation. 
Additionally, \cite{gopalan2015thompson} provides regret guarantees for TS in the finite and infinite MDP setting. Another notable paper is \cite{gopalan2014thompson}, which studies the stochastic MAB problem in complex action settings providing a regret bound that scales logarithmically in time with improved constants.  None of the aforementioned papers study the performance of TS for linear bandits with safety constraints. %Our proof for  Safe-LTS is inspired by the proof technique in \cite{abeille2017linear}.

\section{Safe Linear Thompson Sampling}

Our proposed algorithm is a safe variant of Linear Thompson Sampling (LTS).
At any round $t$, given a regularized least-squares (RLS) estimate $\hat{\theta_t}$, the algorithm samples a perturbed parameter $\Tilde{\theta}_t$ that is appropriately distributed to guarantee sufficient exploration. Considering this sampled $\Tilde{\theta}_t$ as the true environment, the algorithm chooses the action with the highest possible reward while making sure that the safety constraint \eqref{safty.const} holds. 
%The presence of the safety constraint complicates the learner's choice of actions. 
In order to ensure that actions remain safe at all rounds, the algorithm uses the side-information \eqref{eq:side} to construct a confidence region $\mathcal{C}_t$, which  contains the unknown parameter $\mu_{\star}$ with high probability. With this, it forms an \emph{inner} approximation $\mathcal{D}_t^s$ of the safe set, which is composed by all actions $x_t$ that satisfy the safety constraint \emph{for all} $v \in \mathcal{C}_t$.  The summary is presented in Algorithm \ref{alg:safe.Ts} and a detailed description follows. % in the rest of this section. 

\begin{algorithm}[tb]
   \caption{Safe Linear Thompson Sampling (Safe-LTS)} 
\begin{algorithmic}
   \STATE {\bfseries Input:} $\delta, T, \lambda$
%   \REPEAT
   \STATE Set $\delta' = \frac{\delta}{6}$
   \FOR{$t=1$ {\bfseries to} $T$}
%   \IF{$x_i > x_{i+1}$}
   \STATE Sample $\eta_t \sim \mathcal{H}^{\text{TS}}$ (see Section \ref{challenges of safety})  
   \STATE Set $V_t = \lambda I + \sum_{s=1}^{t-1} x_s x_s^{\top}$ 
%   \ENDIF
    \STATE Compute RLS-estimate $\hat{\theta}_t$ and $\hat{\mu}_t$ (see \eqref{RLS-estimate})
    \STATE Set:  $\Tilde{\theta}_t = \hat{\theta}_t + \beta_t(\delta') V_t^{-\frac{1}{2}} \eta_t$
    \STATE Build the confidence region:\\ $\mathcal{C}_t (\delta') = \{ v \in \mathbb{R^d} : \norm{ v-\hat{\mu}}_{V_t} \leq \beta_t(\delta')   \}$
    \STATE Compute the estimated safe set:\\ $\mathcal{D}_t^s = \{ x \in \mathcal{D}_0 : x^{\top}  v \leq C, \forall v \in \mathcal{C}_t (\delta')   \}$ 
    \STATE Play the following safe action: \\ ~ $x_t = \arg\max_{x \in \mathcal{D}_t^s}  x^{\top} \Tilde{\theta}_t$ 
    \STATE Observe reward $r_t$ and measurement $w_t$ 
   \ENDFOR
%   \UNTIL{$noChange$ is $true$}
\label{alg:safe.Ts}
\end{algorithmic}
\end{algorithm}

\subsection{Model assumptions}
\textbf{Notation.} $[n]$ denotes the set $\{1,2,\dots,n \}$. The Euclidean norm of a vector $x$ is denoted by $\norm{x}_2$. Its weighted $\ell_2$-norm with respect to a positive semidefinite matrix $V$ is denoted by $\norm{x}_V = \sqrt{x^{\top} V x}$. We also use the standard $\Otilde$ notation that ignores poly-logarithmic factors. Finally, for ease of notation, from now onwards we refer to the safe set in \eqref{safe.decision.set} by $\mathcal{D}_0^s$ and drop the dependence on $\mu_{\star}$.

Let  $\mathcal{F}_t = (\mathcal{F}_1, \sigma(x_1,\dots,x_t, \xi_1,\dots,\xi_t, \zeta_1,\dots,\zeta_t))$ denote the filtration that represents the accumulated information up to round $t$. In the following, we introduce standard assumptions on the problem.

\begin{assumption} \label{assmp.1}
For all $t$, $\xi_t$ and $\zeta_t$ are conditionally zero-mean, $R$-sub-Gaussian noise variables, i.e., $\mathbb{E}[\xi_t | \mathcal{F}_{t-1}] = \mathbb{E}[\zeta_t | \mathcal{F}_{t-1}] = 0$, and $\mathbb{E}[e^{\alpha \xi_t} | \mathcal{F}_{t-1}] \leq \exp{(\frac{\alpha^2 R^2}{2})}$, $\mathbb{E}[e^{\alpha \zeta_t} | \mathcal{F}_{t-1}] \leq \exp{(\frac{\alpha^2 R^2}{2})}, \forall \alpha \in \mathbb{R}^d$.
\end{assumption}

\begin{assumption}\label{assmp.2}
There exists a positive constant $S$ such that $ \norm{ \theta_{\star} }_2 \leq S$ and $ \norm{ \mu_{\star} }_2 \leq S$.
\end{assumption}

\begin{assumption}\label{assmp.3}
The action set $\mathcal{D}_0$ is a compact and convex subset of $\mathbb{R}^d$ that contains the origin. We assume $\norm{x}_2 \leq L$, $\forall x \in \mathcal{D}_0$.
\end{assumption}

% \begin{remark}
It is straightforward to generalize our results when the sub-gaussian constants of $\xi_t$ and $\zeta_t$ and/or the upper bounds on $\|\theta_{\star}\|_2$ and $\|\mu_{\star}\|_2$ are different. Throughout, we assume they are equal, for brevity. 
%for brevity of the notation we assume these are equal. 
% \end{remark}

\subsection{Algorithm description and discussion} Let $(x_1,\dots,x_t)$ be the sequence of actions and $(r_1,\dots,r_t)$ and $(w_1,\dots,w_t)$ be the corresponding rewards and side-information measurements, respectively. For any $\lambda > 0$, we can obtain RLS-estimates $\hat{\theta}_t$ of $\theta_{\star}$ and $\hat{\mu}_t$ of $\mu_{\star}$ as follows: \begin{align}
    & \hat{\theta}_t = V_t^{-1} \sum_{s=1}^{t-1} r_s x_s, ~~\hat{\mu}_t = V_t^{-1} \sum_{s=1}^{t-1} w_s x_s, \label{RLS-estimate}\\
    & V_t = \lambda I + \sum_{s=1}^{t-1} x_s x_s^{\top}, \label{ Gram.matrix}
\end{align}    
where  $V_t$ is the Gram matrix of the actions.
Based on $\hat{\theta}_t$ and $\hat{\mu}_t$, Algorithm \ref{alg:safe.Ts} constructs two confidence regions $\mathcal{E}_t:=\mathcal{E}_t(\delta')$ and $\mathcal{C}_t:=\mathcal{C}_t(\delta')$  as follows: \begin{align}
& \mathcal{E}_t := \{ \theta \in \mathbb{R}^d : \norm{\theta-\hat{\theta}_t}_{V_t} \leq \beta_t(\delta') \}, \label{ellipsoid.RLS} \\&  \mathcal{C}_t := \{ v \in \mathbb{R}^d : \norm{v-\hat{\mu}_t}_{V_t} \leq \beta_t(\delta') \}.
\end{align}
Notice that $\mathcal{E}_t$ and  $\mathcal{C}_t$ both depend on $\delta'$, but we suppress notation for simplicity, when clear from cotext.
The ellipsoid radius $\beta_t$ is chosen according to the Theorem \ref{abbasi.2} in \cite{abbasi2011improved} in order to guarantee that $\theta_{\star} \in \mathcal{E}_t  $ and $\mu_{\star} \in \mathcal{C}_t$ with high probability.

\begin{theorem}\label{abbasi.2}\cite{abbasi2011improved}
Let Assumptions \ref{assmp.1} and \ref{assmp.2} hold. For a fixed $\delta \in (0,1)$, and \begin{align}\label{eq:beta_t}
    &\beta_t(\delta) = R \sqrt{d \log{\Big(\frac{1 + \frac{(t-1)L^2}{\lambda}}{\delta}\Big)}} + \sqrt{\lambda} S,  
\end{align} with probability at least $1-\delta$, it holds that  $\theta_{\star} \in \mathcal{E}_t(\delta) $ and $\mu_{\star} \in \mathcal{C}_t(\delta)$, for all $t \geq 1$.
\end{theorem}

\subsubsection{Background on LTS: a frequently optimistic algorithm}
Our algorithm inherits the frequentist view of LTS first introduced in \cite{agrawal2013thompson,abeille2017linear}, which is essentially defined as a \emph{randomized} algorithm over the RLS-estimate $\hat{\theta}_t$ of the unknown parameter $\theta_{\star}$. Specifically, at any round $t$, the randomized algorithm of \cite{agrawal2013thompson,abeille2017linear} samples a parameter $\Tilde{\theta}_t$ centered at $\hat{\theta}_t$:
\begin{align}
    \Tilde{\theta}_t = \hat{\theta}_t + \beta_t(\delta') V_t^{-\frac{1}{2}} \eta_t, \label{theta.tilde}
\end{align} 
and chooses the action $x_t$ that is best with respect to the new sampled parameter, i.e., maximizes the objective $x_t^{\top}\Tilde{\theta}_t$. The key idea of \cite{agrawal2013thompson,abeille2017linear} on how to select the random perturbation $\eta_t\in \mathbb{R}^d$ to guarantee good regret performance is as follows. On the one hand, $\Tilde{\theta}_t$ must stay close enough to the RLS-estimate $\hat{\theta}_t$ so that $x_t^{\top}\Tilde{\theta}_t$ is a good proxy for the true (but unknown) reward $x_t^{\top}{\theta}_\star$. Thus, $\eta_t$ must satisfy an appropriate \emph{concentration} property. On the other hand, $\Tilde{\theta}_t$ must also favor exploration in a sense that it leads --often enough-- to  actions $x_t$ that are \emph{optimistic}, i.e., they satisfy 
\begin{align}\label{eq:optimism}
    x_t^{\top}\Tilde{\theta}_t\geq x_\star^{\top}\theta_\star
\end{align}
Thus, $\eta_t$ must satisfy an appropriate \emph{anti-concentration} property.

Our proposed Algorithm \ref{alg:safe.Ts} also builds on these two key ideas. However, we discuss next how the safe setting imposes additional challenges and how our algorithm and its analysis manages to address them.

\subsubsection{Addressing challenges in the safe setting} \label{challenges of safety}

Compared to the classical linear bandit setting studied in \cite{agrawal2013thompson,abeille2017linear}, the presence of the safety constraint raises the following two questions: 

\noindent \emph{(i)} How to guarantee actions played at each round are safe? 

\noindent \emph{(ii)} In the face of the safety restrictions, how can optimism (cf. \eqref{eq:optimism}) be maintained?

In the rest of this section, we explain the mechanisms that Safe-LTS  Algorithm \ref{alg:safe.Ts} employs to address both of these challenges.

\noindent\textbf{Safety -}~First, the chosen action $x_t$ at each round need not only maximize $x_t^{\top}\Tilde{\theta}_t$, but also, it needs to be safe. Since the learner does not know the safe action set $\mathcal{D}_0^s$, Algorithm \ref{alg:safe.Ts} performs conservatively and guarantees safety as follows. After creating the confidence region $\mathcal{C}_t$ around the RLS-estimate $\hat{\mu}_t$, it forms the so-called {\it safe decision set at round $t$} denoted as $\mathcal{D}_t^s$:
\begin{align}
    \mathcal{D}_t^s = \{ x \in  \mathcal{D}_0 : x^{\top} v \leq C, \forall v \in \mathcal{C}_t \}. \label{safe.decision.set}
\end{align}
Then, the chosen action is optimized over only the subset $\mathcal{D}_t^s$, i.e.,
\begin{align}
    x_t = \arg\max_{x \in \mathcal{D}_t^s} x^{\top} \Tilde{\theta}_t. \label{played.action}
\end{align}
We make the following two important remarks about the set $\mathcal{D}_t^s$ defined in \eqref{safe.decision.set}. On a positive note, $\mathcal{D}_t^s$ is easy to compute. To see this, note the following equivalent definitions:
\begin{align}
    {\mathcal{D}}_t^s & := \{ x \in \mathcal{D}_0 : x^{\top} v \leq C , \forall v \in {\mathcal{C}}_t \} \nonumber\\&= \{x \in \mathcal{D}_0 : \max_{v \in {\mathcal{C}}_t} x^{\top} v \leq C  \} \nonumber\\& = \{x \in \mathcal{D}_0 : x^{\top} \hat{\mu}_t + \beta_t(\delta') \norm{x}_{V_t^{-1}} \leq C  \}.  \label{quadratic.const}
\end{align}
Thus, in view of \eqref{quadratic.const}, the optimization in \eqref{played.action} is an efficient convex quadratic program. The challenge with $\mathcal{D}_t^s$ is that it contains actions which are safe with respect to \emph{all} the parameters in $\mathcal{C}_t$, and not only $\mu_{\star}$. As such, it is only an \emph{inner} approximation of the true safe set $\mathcal{D}_0^s$. As we will see next, this fact complicates the requirement for optimism.

\noindent\textbf{Optimism in the face of safety -}~As previously discussed, in order to guarantee safety, Algorithm \ref{alg:safe.Ts} chooses actions from the subset $\mathcal{D}_t^s$. This is only an inner approximation of the true safe set $\mathcal{D}_0^s$, a fact that makes it harder to maintain optimism of $x_t$ as defined in \eqref{eq:optimism}. To see this, note that in the classical setting of \cite{abeille2017linear}, their algorithm would choose $x_t$ as the action that maximizes $\Tilde{\theta}_t$ over the \emph{entire} set $\mathcal{D}_0$. In turn, this would imply that $x_t^{\top}\Tilde{\theta}_t\geq x_\star^{\top}\Tilde{\theta}_t$ because $x_\star$ belongs to the feasible set $\mathcal{D}_0$. This  observation is the critical first argument in proving that $x_t$ is optimistic often enough, i.e., \eqref{eq:optimism}  holds with fixed probability $p>0$. 

Unfortunately, in the presence of safety constraints, the action $x_t$ is a maximizer over only the subset $\mathcal{D}_t^s$. Since $x_{\star}$ may \emph{not} lie within $\mathcal{D}_t^s$, there is no guarantee that $x_t^{\top}\Tilde{\theta}_t\geq x_\star^{\top}\Tilde{\theta}_t$ as before. So, how does then one guarantee optimism? 

Intuitively, at the first rounds, the estimated safe set $\mathcal{D}_t^s$ is only a small subset of the true $\mathcal{D}_0^s$. Thus, $x_t \in \mathcal{D}_t^s$ is a vector of small norm compared to that of $x_\star\in\mathcal{D}_0^s$. Thus, for \eqref{eq:optimism} to hold, it must be that $\Tilde{\theta}_t$ is not only in the direction of $\theta_\star$, but it also has larger norm than that. To satisfy this latter requirement, the random vector $\eta_t$ must be large; hence, it will ``anti-concentrate more". As the algorithm progresses, and --thanks to side-information measurements-- the set $\mathcal{D}_t^s$ becomes an increasingly better approximation of $\mathcal{D}_0^s$, the requirements on anti-concentration of $\eta_t$ become the same as if no safety constraints were present. Overall, at least intuitively, we might hope that optimism is possible in the face of safety, but only provided that $\eta_t$ is set to satisfy a stronger (at least at the first rounds) anti-concentration property than that required by \cite{abeille2017linear} in the classical setting.

At the heart of Algorithm \ref{alg:safe.Ts} and its proof of regret lies an analytic argument that materializes the intuition described above. Specifically, we will prove that optimism is possible in the presence of safety at the cost of a stricter anti-concentration property compared to that specified in \cite{abeille2017linear}. While the proof of this fact is deferred to Section \ref{sketch.of.the.proof}, we now summarize the appropriate distributional properties that provably guarantee good regret performance of Algorithm \ref{alg:safe.Ts} in the safe setting.

\begin{definition}[] \label{definition.noise}
In Algorithm \ref{alg:safe.Ts}, the random vector $\eta_t$ is sampled IID at each round $t$ from  a multivariate distribution $\mathcal{H}^{\text{TS}}$ on $\mathbb{R}^{d}$ that is absolutely continuous with respect to the {\it Lebesgue} measure and satisfies the following properties:\\
\noindent~~\emph{Anti-concentration:} There exists a strictly positive probability $p>0$ such that for any $u \in \mathbb{R}^d$ with $\norm{u}_2 = 1$, 
    \begin{align}
        \mathbb{P} \big( u^{\top} \eta_t \geq 1 + \frac{2}{C} LS  \big) \geq p. \label{anti.concen}
     \end{align}
     
\noindent\emph{Concentration:} There exists positive constants $c,c'>0$ such that $\forall \delta \in (0,1)$, 
    \begin{align}
        \mathbb{P} \big(  \norm{ \eta_t}_2 \leq \big( 1 + \frac{2}{C} LS \big)  & \sqrt{c d \log{(\frac{c' d}{\delta})}} \,\big)  \geq 1 - \delta \label{concen.}.
    \end{align}
\end{definition}

In particular, the difference to the distributional assumptions required by \cite{abeille2017linear} in the classical setting is the extra term $\frac{2}{C}LS$ in \eqref{anti.concen} (naturally, the same term affects the concentration property \eqref{concen.}).
%We will see in Section \ref{sketch.of.the.proof} that this extra term $\frac{2}{C}LS$ is critical (and non-trivial) in order to provide the regret guarantee for the more challenging problem of LTS with side constraints. 

Our proof of regret in Section \ref{regret analysis} shows that this extra term captures an appropriate notion of the distance between the approximation $\mathcal{D}_t^s$ (where $x_t$ lives) and the true safe set $\mathcal{D}_0^s$ (where $x_\star$ lives), and provides enough exploration for the sampled parameter $\Tilde{\theta}_t$ so that actions in $\mathcal{D}_t^s$ can be optimistic. While this intuition can possibly explain the need for an additive term in Definition \ref{definition.noise}, it is insufficient when it comes to determining what the ``correct" value for it should be. This is determined by our analytic treatment in Section \ref{sketch.of.the.proof}.

As a closing remark of this section, note that the properties in Definition \ref{definition.noise} are satisfied by a number of easy-to-sample from distributions. For instance, it is satisfied by a multivariate zero-mean IID Gaussian distribution with all entries having a (possibly time-dependent) variance $1+ \frac{2}{C}LS$. 

\section{Regret Analysis}\label{regret analysis}

In this section, we present our main result, a tight regret bound for Safe-LTS, and discuss key proof ideas. 

Since $\mu_{\star}$ is unknown, the learner does not know the safe action set $\mathcal{D}_0^s$. Therefore, in order to satisfy the safety constraint \eqref{safty.const},  Algorithm \ref{alg:safe.Ts} chooses actions from $\mathcal{D}_t^s$, which is a conservative inner approximation of $\mathcal{D}_0^s$. Moreover, at the heart of the action selection rule, is the sampling of an appropriate random perturbation $\Tilde\theta_t$ of the RLS estimate $\hat\theta_t$. The sampling rule in \eqref{theta.tilde} is almost the same as in the classical setting \cite{abeille2017linear}, but with a key difference on the distribution $\mathcal{H}^{TS}$ of the perturbation vector $\eta_t$ (cf. Definition \ref{definition.noise}). As explained in Section \ref{sketch.of.the.proof}, the modification compared to \cite{abeille2017linear} is necessary to guarantee that actions are frequently optimistic (see \eqref{optimisitc}) in spite of limitations on actions imposed because of the safety constraints. 
In this section, we put these pieces together by proving that the action selection rule of Safe-LTS  is simultaneously: 1) frequently optimistic, and, 2) guarantees a proper expansion of the estimated safe set. Our main result stated as Theorem \ref{Main.theorem} is perhaps surprising: in spite of the additional safety constraints, Safe-LTS has regret $\mathcal{O}(\sqrt{T} \log^{\frac{3}{2}} T)$ that is order-wise the same as that in the classical setting \cite{agrawal2013thompson,abeille2017linear}.

\begin{theorem}[Regret of Safe-LTS] \label{Main.theorem}
Let $\lambda \geq 1$.  Under Assumptions \ref{assmp.1}, \ref{assmp.2}, \ref{assmp.3}, the regret $R(T)$ of Safe-LTS Algorithm \ref{alg:safe.Ts} is upper bounded with probability $ 1- \delta$ as follows: 
  \begin{align}
      R(T)&\leq \big( \beta_T(\delta') + \gamma_T(\delta') (1+\frac{4}{p}) \big) \sqrt{2 T d \log{ (1 + \frac{T L^2}{\lambda})}} \nonumber \\& + \frac{4 \gamma_T(\delta')}{p} \sqrt{\frac{8T L^2}{\lambda} \log{\frac{4}{\delta}}},
  \end{align} where $\delta' = \frac{\delta}{6T}$, $\beta_t(\delta')$ as in \eqref{eq:beta_t} and, 
  %\begin{align}
    $  \gamma_t(\delta') = \beta_t(\delta') \big(  1 + \frac{2}{C}LS \big) \sqrt{c d \log{(\frac{c' d}{\delta})}}\,. $
      %\label{gamma}
  %\end{align}
  \end{theorem}

A detailed proof of Theorem \ref{Main.theorem} is deferred to Appendix \ref{appdnx.proof.of.theorem}. In the rest of the section, we highlight the key changes compared to previous proofs in \cite{agrawal2013thompson,abeille2017linear} that occur due to the safety constraint. To begin, let us  consider the following standard decomposition of the cumulative regret $R(T)$: 
\begin{align}
  %& R(T)   \leq  
  %\sum_{t=1}^T \left(x^{\top}_{\star} \theta_{\star} - x_t^{\top} \theta_{\star}\right) = 
  %\nonumber\\&
  \sum_{t=1}^{T} \big( \underbrace{ x_{\star}^{\top}\theta_{\star} - x_t^{\top}\Tilde{\theta}_t}_{\text{Term I}} \big)  +  \sum_{t=1}^T \big( \underbrace{ x_t^{\top}\Tilde{\theta}_t - x_t^{\top} \theta_{\star}}_{\text{Term II}} \big). \label{overall.regret}
\end{align} 
Regarding Term II, the concentration property of $\mathcal{H}^{TS}$ guarantees that $\Tilde{\theta}_t$ is close to $\hat{\theta}_t$, and consequently, close to $\theta_\star$ thanks to Theorem \ref{abbasi.2}. Therefore, controlling Term II can be done similar to previous works e.g., \cite{abbasi2011improved,abeille2017linear}; see Appendix \ref{appndix.bound.term II} for more details. Next, we focus on Term I.

To see how the safety constraints affect the proofs let us first review the treatment of Term I in the classical setting. For UCB-type algorithms, Term I is always non-positive since the pair $(\Tilde{\theta}_t,x_t)$ is optimistic at each round $t$ by design \cite{Dani08stochasticlinear,Tsitsiklis,abbasi2011improved}. For LTS, Term I can be positive; that is, \eqref{eq:optimism} may not hold at every round $t$. However, \cite{agrawal2013thompson,abeille2017linear} proved that thanks to the anti-concentration property of $\eta_t$, this optimistic property occurs often enough. Moreover, this is enough to yield a good enough bound on Term I for \emph{every} round $t$; see Appendix \ref{sec:sketch_app}.

As discussed in Section \ref{challenges of safety}, the requirement for safety complicates the requirement for optimism. Our main technical contribution, detailed in the next section, is to show that the properly modified anti-concentration property in Definition \ref{definition.noise} together with the construction of approximated safe sets as in \eqref{quadratic.const} can yield frequently optimistic actions even in the face of safety. Specifically, it is the extra term $\frac{2}{C}LS$ in \eqref{anti.concen} that allows enough  exploration to the sampled parameter $\Tilde{\theta}_t$ in order to compensate for safety limitations on the chosen actions, and because of that we are able to show  Safe-LTS obtains the same order of regret as that of \cite{abeille2017linear}.

 %%%%%%%%%%%%%%%%%%%%%%%%%%%%%%%%%%%%%%%%%%%%%%%%%%%%%%%%%%%%
\subsection{Proof sketch: Optimism despite safety constraints}\label{sketch.of.the.proof}
%%%%%%%%%%%%%%%%%%%%%%%%%%%%%%%%%%%%%%%%%%%%%%%%%%%%%%%%%%%%%%%

We prove that Safe-LTS samples a parameter $\Tilde{\theta}_t$ that is optimistic with constant probability. The next lemma informally characterizes this claim (see the formal statement of the lemma and its proof in Appendix \ref{appndx.proof.of.lema}).
\begin{lemma}(Optimism in the face of safety; Informal)\label{optimitic.lemma}
For any round $t\geq 1$, Safe-LTS samples a parameter $\Tilde{\theta}_t$ and chooses an action $x_t$ such that the pair $(\Tilde{\theta}_t,x_t)$ is optimistic frequently enough, i.e.,
 \begin{align}
    \mathbb{P} \left(x_t^{\top} \Tilde{\theta}_t \geq x_{\star}^{\top} \theta_{\star} \right) \geq p, \label{optimitic.frequency} 
\end{align}
where $p>0$ is the probability with which the anti-concentration property \eqref{anti.concen} holds.
\end{lemma}

The challenge in the proof is that  the actions $x_t$ are chosen from the estimated safe set $\mathcal{D}_t^s$, which does not necessarily contain all feasible actions and hence, {\it may not contain $x_{\star}$}. Therefore, we need a mechanism to control the distance of the optimal action $x_{\star}$ from the optimistic actions that can only lie within the subset $\mathcal{D}_t^s$ ({\it distance} is defined here in terms of an inner product with the optimistic parameters $\Tilde{\theta}_t$).
Unfortunately, we do not have a direct control on this distance term and so at the heart of the proof lies the idea of identifying a ``good'' feasible action $\tilde{x}_t\in\mathcal{D}_t^s$ whose distance to $x_{\star}$ is easier to control.

To be concrete, we show that it suffices to choose the good feasible point in the direction of $x_\star$, i.e.,  $\tilde{x}_t=\alpha_t x_\star$, where the key parameter $\alpha_t \in (0,1]$ must be set to satisfy $\tilde{x}_t\in\mathcal{D}_t^s$.  Naturally, the value of $\alpha_t$ is determined by the approximated safe set $\mathcal{D}_t^s$ as defined in \eqref{quadratic.const}. The challenge though is that we do not know how the value of $x^{\top}_*\hat{\mu}_t$ compares to the constant $C$. We circumvent this issue by introducing an enlarged confidence region centered at $\mu_{\star}$ as \begin{align}
    \Tilde{\mathcal{C}}_t  := \{ v \in \mathbb{R}^d : \norm{v - \mu_{\star}}_{V_t} \leq 2 \beta_t(\delta')  \}, \nonumber
\end{align}and the corresponding shrunk safe decision set as 
\begin{align}
    \Tilde{\mathcal{D}}_t^s & := \{ x \in \mathcal{D}_0 : x^{\top} v \leq C , \forall v \in \Tilde{\mathcal{C}}_t \}   \label{quadratic.const.shrunk}\\& = \{x \in \mathcal{D}_0 : x^{\top} \mu_{\star} + 2 \beta_t(\delta') \norm{x}_{V_t^{-1}} \leq C  \} \subseteq {\mathcal{D}}_t^s .\nonumber
\end{align}
Notice that the shrunk safe set is defined with respect to an ellipsoid centered at $\mu_{\star}$ (rather than at $\hat{\mu}_t$). This is convenient since $x^{\top}_{\star}\mu_{\star}\leq C $. Using this, it can be easily checked that the following choice of $\alpha_t$:
\begin{align}
    {\alpha_t} = \big(1 + \frac{2}{C}\beta_t(\delta') \norm{x_{\star}}_{V_t^{-1}}\big)^{-1}, \label{alpha}
\end{align}
ensures that 
$$
\alpha_t x_{\star} \in \Tilde{\mathcal{D}}_t^s \subseteq {\mathcal{D}}_t^s.
$$
From this, and optimality of $x_t=\arg\max_{x\in\mathcal{D}_t^s}x^{\top}\Tilde{\theta}_t$ we have that: 
\begin{align}
    x_t^{\top} \Tilde{\theta}_t \geq \alpha_t x_{\star}^{\top} \Tilde{\theta}_t \label{sketch.optimistic.projection}.
\end{align} 

Next, using \eqref{sketch.optimistic.projection}, in order to show that \eqref{optimitic.frequency} holds, it suffices to prove  that \begin{align}
    p &\leq \mathbb{P} \big( \alpha_t x_{\star}^{\top} \Tilde{\theta}_t \geq  x_{\star}^{\top} \theta_{\star} \big) \nonumber\\& = \mathbb{P} \big( x_{\star}^{\top} \Tilde{\theta}_t \geq  x_{\star}^{\top} \theta_{\star} + \frac{2}{C} \beta_t(\delta') \norm{x_{\star}}_{V_t^{-1}} x_{\star}^{\top} \theta_{\star} \big), \nonumber
\end{align}
where, in the second line, we used the definition of $\alpha_t$ in \eqref{def.alpha}. To continue, recall that $\Tilde{\theta}_t = \hat{\theta}_t + \beta_t V_t^{-\frac{1}{2}} \eta_t$.  Thus, we want to lower bound the probability of the following event:
%the  probability we want to lower bound can be equivalently rewritten as 
\begin{align}\nonumber
 %\mathbb{P} \big(
 \beta_t(\delta') x^{\top}_{\star}    V_t^{-\frac{1}{2}} &  \eta_t   \geq x^{\top}_{\star} (\theta_{\star}-\hat{\theta}_t) +  \frac{2}{C} \beta_t(\delta') \norm{x_{\star}}_{V_t^{-1}} x^{\top}_{\star} \theta_{\star}   
 %\big).  \nonumber
\end{align} 
To simplify the above, we use the following two facts: (i) $|x_{\star}^{\top} \theta_{\star}| \leq \|x_{\star}\|_2 \|\theta_{\star}\|_2 \leq LS $; (ii) $ x_{\star}^{\top}  ( \theta_{\star} - \hat{\theta}_t)\leq \|x_{\star}\|_{V_t^{-1}}\|\theta_{\star} - \hat{\theta}_t\|_{V_t} \leq \beta_t(\delta')\|x_{\star}\|_{V_t^{-1}}$, because of Cauchy-Schwartz and Theorem \ref{abbasi.2}
Put together, we need that 
\begin{align}
   p \leq \mathbb{P} \big( \beta_t(\delta') x_{\star} V_t^{-\frac{1}{2}} \eta_t   \geq & \beta_t(\delta')\|x_{\star}^{\top}\|_{V_t^{-1}} + \nonumber \\& \frac{2}{C}LS\beta_t(\delta') \norm{x_{\star}}_{V_t^{-1}}  \big), \nonumber
\end{align}
or equivalently,
\begin{align}
   p \leq \mathbb{P} \big(  u_t^\top \eta_t   \geq   1 + \frac{2}{C}LS\big), \label{eq:finally}
\end{align}
where we have defined $u_t = \frac{ V_t^{-\frac{1}{2}}x_t}{\norm{x_{\star}}_{V_t^{-1}}}$. By definition of $u_t$, note that $\norm{u_t}_2 = 1$. Hence, the desired \eqref{eq:finally} holds due to the anti-concentration property of the $\mathcal{H}^{\text{TS}}$ distribution in \eqref{anti.concen}. This completes the proof of Lemma \ref{optimitic.lemma}.

Before closing, we remark on the following differences to the proof of optimism in the classical setting as presented in Lemma 3 of \cite{abeille2017linear}. First, we present an algebraic version of the basic machinery introduced in Section 5 of \cite{abeille2017linear} that we show is convenient to extend to the safe setting. Second, we employ the idea of relating $x_t$ to a ``better" feasible point $\alpha_t x_t$ and show optimism for the latter. Third, even after introducing $\alpha_t$, the fact that $1/\alpha_t - 1$ is proportional to $\|x_{\star}\|_{V_t^{-1}}$ (see \eqref{def.alpha}) is critical for the seemingly simple algebraic steps that follow \eqref{alpha}. In particular, in deducing \eqref{eq:finally} from the expression above, note that we have divided both sides in the probability term by $\|x_\star\|_{V_t{-1}}$. It is only thanks to the proportionality observation that we made above that the term $\|x_\star\|_{V_t{-1}}$ cancels throughout and we can conclude with 
 \eqref{eq:finally} without a need to lower bound the minimum eigenvalue of the Gram matrix $V_t$ (which is known to be hard).

\section{Numerical Results and Comparison to State of the Art}\label{sec:simulations}

We present details of our numerical experiments on synthetic data. First, we show how the presence of safety constraints  affects the performance of LTS in terms of regret. Next, we evaluate Safe-LTS by comparing it against safe versions of LUCB. Then, we compare Safe-LTS to \cite{amani2019linear}'s Safe-LUCB.
In all the implementations, we used: $T = 10000 , \delta = 1/4T$, $R=0.1$ and $\mathcal{D}_0 = [-1,1]^4$. Unless otherwise specified, the reward and constraint parameters $\theta_\star$ and $\mu_\star$ are drawn from $\mathcal{N}(0,I_4)$; $C$ is drawn uniformly from $[0,1]$. Throughout, we have implemented a modified version of Safe-LUCB which uses $\ell_1$-norms instead of $\ell_2$-norms, due to computational considerations (e.g., \cite{Dani08stochasticlinear,amani2019linear}). Recall that the action selection rule of UCB-based algorithms involves solving bilinear optimization problems, whereas, TS-based algorithms (such as the one proposed here) involve simple linear objectives (see \cite{abeille2017linear}). 
%
%
% This highlights a well-known benefit associated with TS-based algorithms, namely that they are easier to implement and more computationally-efficient than UCB-based algorithms. In particular, the action selection rule in UCB-based algorithms involves solving optimization problems with bilinear objective functions, whereas, for TS-based algorithms, it would lead to linear objectives (see \cite{abeille2017linear}). 
% by first sampling the unknown parameter vector

\subsection{The  effect of safety constraints on LTS}

In Fig. \ref{regret} we compare the average cumulative regret of Safe-LTS to the standard LTS algorithm with \emph{oracle access} to the true safe set $\mathcal{D}_0^s$. The results are averages over 20 problem realizations. As shown, even though  Safe-LTS requires that chosen actions belong to the conservative  inner-approximation set $\mathcal{D}_t^s$, it still achieves a regret of the same order as the oracle reaffirming the prediction of Theorem \ref{Main.theorem}. Also, the comparison to the oracle reveals that the action selection rule of Safe-LTS is indeed such that it guarantees a fast expansion of the estimated safe set so as to not exclude optimistic actions for a long time. Fig. \ref{regret} also shows the performance of a third algorithm discussed in Sec. \ref{dyn.noise}.

   \begin{figure}
     \centering
          \includegraphics[width=0.65\linewidth]{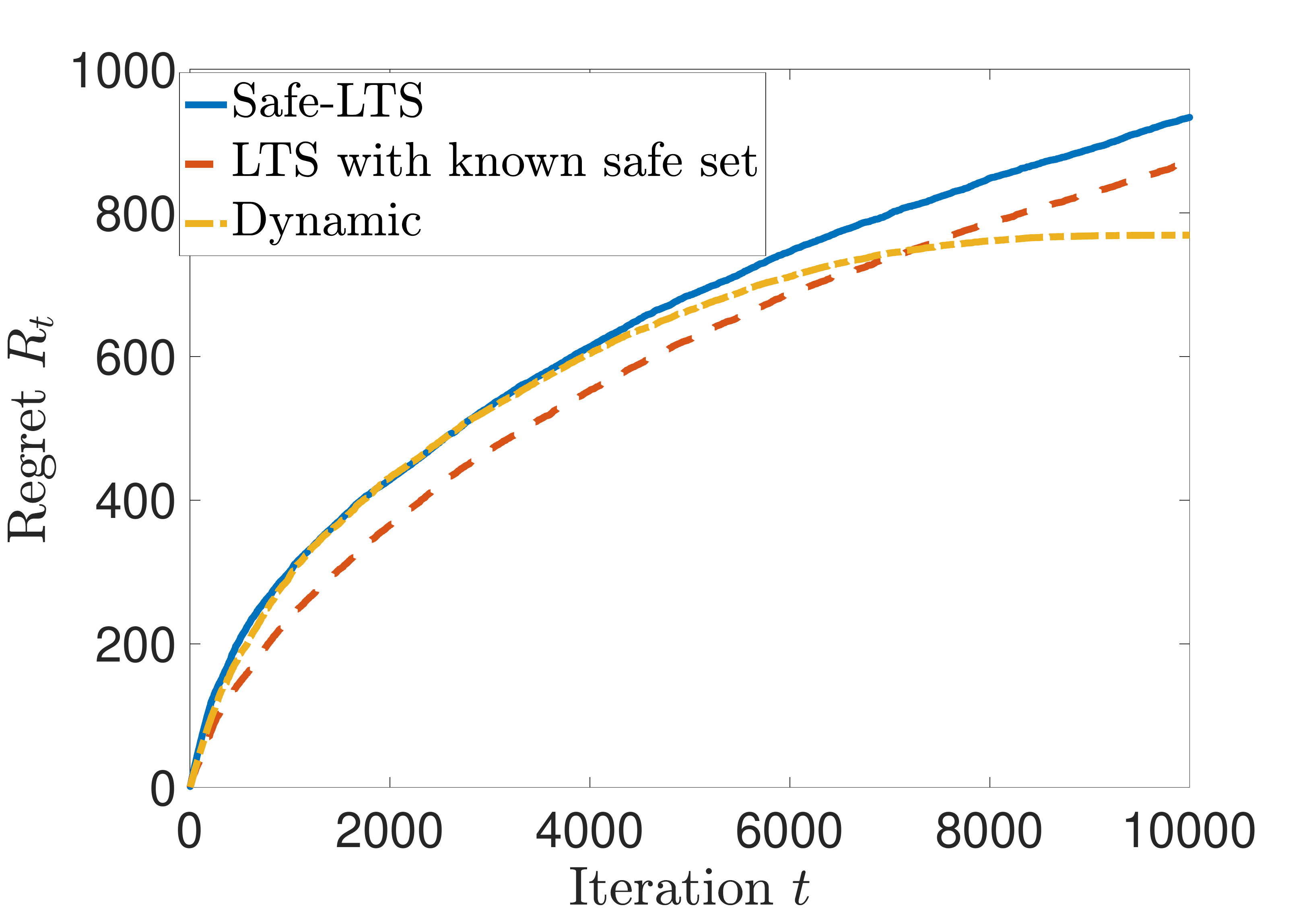}
         \caption{Comparison of the average cumulative regret of Safe-LTS vs standard LTS with oracle access to the safe set and Safe-LTS with a dynamic noise distribution described in Section \ref{dyn.noise}.}
         \label{regret}
   \end{figure}

    \begin{figure}
     \centering
          \includegraphics[width=0.65\linewidth]{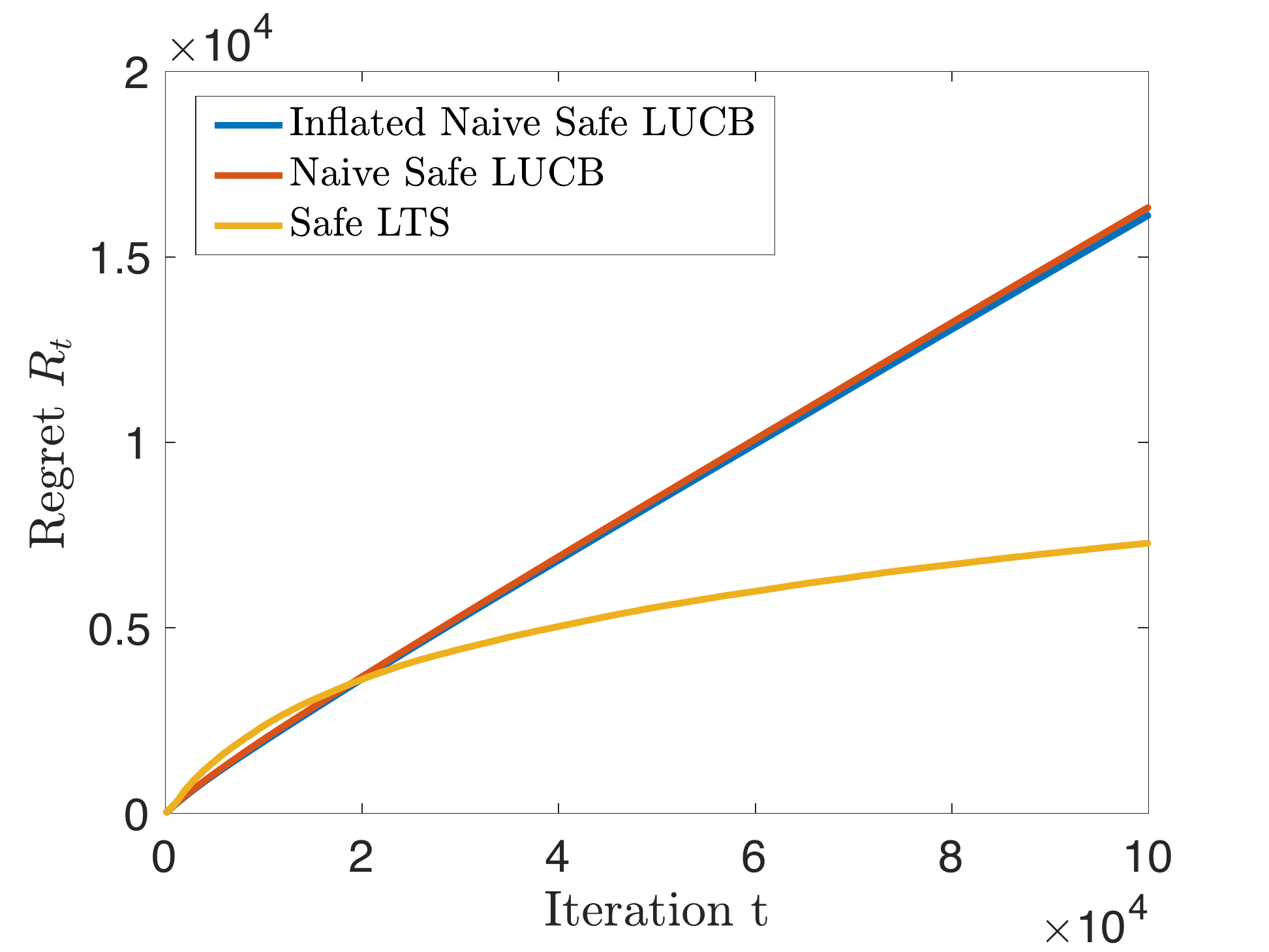}
         \caption{The cumulative regret of Safe-LTS, Naive Safe-LUCB and Inflated Naive Safe-LUCB for a specific problem instance.  }
         \label{inflated_regret}
  \end{figure}

\subsection{Comparison to safe  versions of LUCB}\label{comparison with safe-ucb}
Here, we  compare the performance of our algorithm with two safe versions of LUCB, as follows. First, we implement a natural extension of the classical LUCB algorithm in \cite{Dani08stochasticlinear}, which we call ``Naive Safe-LUCB" and which respects safety constraints by choosing actions from the estimated safe set in \eqref{safe.decision.set}. Second, we consider an improved version, which we call ``Inflated Naive Safe-LUCB" and which is motivated by our analysis of Safe-LTS. Specifically, in light of Lemma \ref{optimitic.lemma}, we implement the improved LUCB algorithm with an inflated  confidence ellipsoid by a fraction $1+\frac{2}{C}LS$ in order to favor optimistic exploration.
In Fig. \ref{inflated_regret}, we employ these two algorithms for a specific problem instance (details in Appendix \ref{simulationssss}) showing that both fail to provide the $\Otilde(\sqrt{T})$ regret of Safe-LTS, in general. Further numerical simulations  suggest that while Safe-LTS always outperforms Naive Safe-LUCB, the Inflated Naive Safe-LUCB can have superior performance to Safe-LTS in many problem instances (see  Fig. \ref{fig:better.inflated} in Appendix \ref{simulationssss}). Unfortunately, not only is this not always the case (cf. Fig. \ref{inflated_regret}), but also we are not aware of an appropriate modification to our proofs to show this problem-dependent performance. This being said further investigations in this direction might be of interest.

%, the favorable performance of the Inflated Naive Safe LUCB algorithm  for some problem instances (see Fig. \ref{fig:better.inflated} in Section \ref{simulationssss} of Appendix) suggests that further study in this direction is warranted.

\subsection{Comparison to Safe-LUCB}\label{camparing with safe-Lucb}
In this section we  compare our algorithm and results to the Safe-LUCB algorithm of \cite{amani2019linear} that was proposed for a similar, but non-identical setting. Specifically, in \cite{amani2019linear}, the linear safety constraint involves  the \emph{same} unknown parameter vector $\theta_\star$ of the linear reward function and --in our notation-- it takes the form $x^{\top} B \theta_{\star} \leq C$, for some \emph{known} matrix $B$. As such, \emph{no} side-information measurements are needed.
First, while our proof does not show a regret of $\Otilde(\sqrt{T})$ for the setting of \cite{amani2019linear} in the general case, it does so for special cases. For example, it is not hard to see that our proofs readily extend to their setting when $B=I$. This already improves upon the $\Otilde(T^{2/3}$)  guarantee provided by \cite{amani2019linear}. Indeed, for  $B=I$, there are non-trivial instances where $C- x_*^{\top}\theta_*=0$ (i.e., the safety constraint is active), in which Safe-LUCB  suffers from a $\Otilde(T^{2/3})$ bound \cite{amani2019linear}. Second, while our proof adapts to a special case of \cite{amani2019linear}'s setting, the other way around is \emph{not} true, i.e., it is not obvious how one would modify the proof of \cite{amani2019linear} to obtain a $\Otilde(\sqrt{T})$ guarantee even in the  presence of side information. This point is highlighted by Fig. \ref{fig:safesets} that numerically compares the two algorithms for a specific problem instance with side information: $\theta_* = [0.9, 0.23]^\top$, $\mu_* = [0.55, 0.31]^T$, and $C = 0.11$ (note that the constraint is  active at the optimal). Also, see Section \ref{simulationssss} for a numerical comparison of the estimated safe-sets' expansion for the two algorithms.
Fig. \ref{fig:regretcomparison} compares Safe-LTS against Safe-LUCB and Naive Safe-LUCB over 30 problem realizations (see Section \ref{simulationssss} in Appendix for plots with standard deviation). As already pointed out in \cite{amani2019linear}, Naive Safe-LUCB generally leads to poor regret, since the LUCB action selection rule alone does not provide sufficient exploration towards safe set expansion. In contrast, Safe-LUCB is equipped with a pure exploration phase over a given seed safe set, which is shown to lead to proper safe set expansion. Our paper reveals that the inherent randomized nature of Safe-LTS is alone capable to properly expand the safe set without the need for an explicit initialization phase (during which regret grows linearly).
%The reader can verify that Naive Safe-LUCB generally leads to a poor regret. For Safe-LUCB, as pointed out in \cite{amani2019linear},
%the LUCB action selection rule alone does not provide sufficient exploration towards safe set expansion, thus requiring the algorithm to 1) have access to a seed safe set; 2) start with a pure exploration phase in order to guarantee safe set expansion. This lack of exploration is especially costly for Safe-LUCB under problem instances where the safety constraint is active.

\begin{figure}
\centering
%   \includegraphics[width=0.32\linewidth]{SL-UCB1.eps}
%     % \caption{Growth of safe sets in Safe-LUCB}   
%   \includegraphics[width=0.32\linewidth]{SL-TS1.eps}
%     %   \caption{Growth of safe sets in Safe-LTS}
\includegraphics[width=0.65\linewidth]{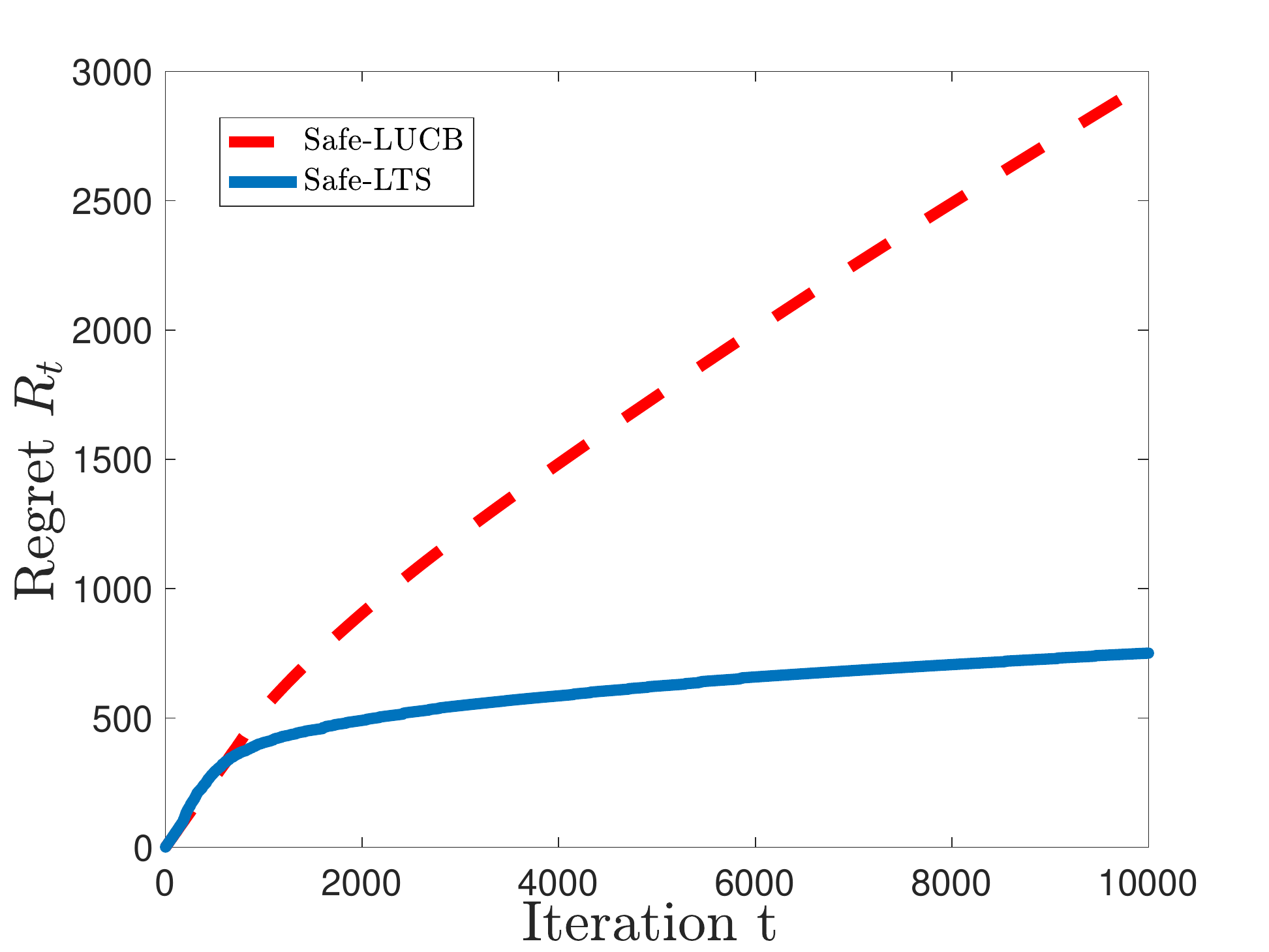}
        \caption{ Comparison of  regret of Safe-LUCB and Safe-LTS, for a single problem instance in which the safety constraint is active.  }
  \label{fig:safesets}
 \end{figure}

%Our numerical experiments in Sections \ref{comparison with safe-ucb} and \ref{camparing with safe-Lucb} suggest that the inherently randomized nature of LTS   is crucial for properly expanding the safe action set. This is exactly where the LUCB decision rule seems to fail in our numerical experiments. Note that the guarantees of \cite{amani2019linear} were obtained only after introducing a randomization phase to LUCB. 

 \begin{figure}
     \centering
          \includegraphics[width=0.65\linewidth]{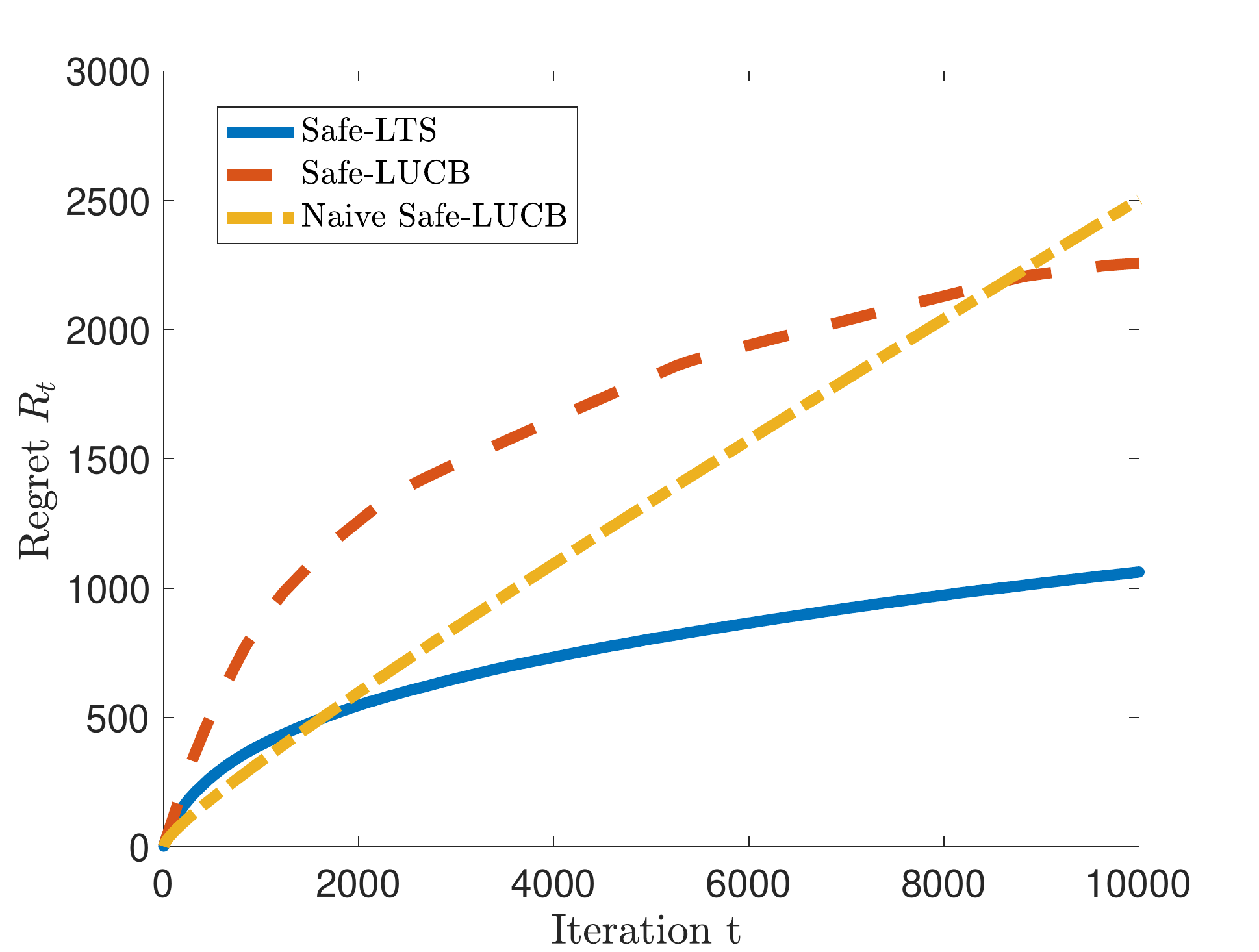}
         \caption{Comparison of the average cumulative regret of Safe-LTS versus two safe LUCB algorithms.}
         \label{fig:regretcomparison}
   \end{figure}
   
\subsection{Sampling from a dynamic noise distribution}\label{dyn.noise}
   
 In order for Safe-LTS to be frequently optimistic, our theory requires that the random perturbation $\eta_t$ satisfies \eqref{anti.concen} {\it for all rounds}. Specifically, we need the extra $\frac{2}{C}LS$ factor compared to \cite{abeille2017linear} in order to ensure safe set expansion. While this result is already sufficient for the tight regret guarantees of Theorem \ref{Main.theorem}, it does not fully capture our intuition (see also Sec.~\ref{challenges of safety}) that as the algorithm progresses and $\mathcal{D}_t^s$ gets closer to $\mathcal{D}_0^s$, exploration (and thus, the requirement on anti-concentration) does not need to be so aggressive. Based on this intuition, we propose the following heuristic modification, in which Safe-LTS uses a perturbation with the following {\it dynamic} property:
% \begin{equation}
        $\mathbb{P}_{\eta \sim \mathcal{H}^{\text{TS}}} \left( u^{\top} \eta \geq k(t)  \right) \geq p,$ %\label{dynamic.anti.concen}
    %  \end{equation} 
for $k(t)$ a linearly-decreasing function $k(t) = (1 + \frac{2}{C}LS)(1-t/T)$. Fig. \ref{regret} shows empirical evidence of the superiority of the heuristic.
 
%  of $\eta_t$ become the same as if no safety constraints were present.
% 
%  only in the initial rounds of the algorithm. As the algorithm progresses and we collect more actions that result in a better estimate of the safe set, explorations do not need to be so aggressive. As such, here we highlight the performance of a heuristic modification of the algorithm  in which the TS distribution $\mathcal{H}^{\text{TS}}$ does not  sample according to the above anti-concentration property for all rounds. We empirically observe that if the TS distribution satisfies \eqref{anti.concen}, the TS algorithm explores more than what it needs to get a good approximation of the unknown parameter, which can cause the growth of the regret. Instead, Fig. \ref{regret} shows that if the TS distribution satisfies the following {\it dynamic} property
%% \begin{equation}
%        $\mathbb{P}_{\eta \sim \mathcal{H}^{\text{TS}}} \left( u^{\top} \eta \geq k(t)  \right) \geq p,$ %\label{dynamic.anti.concen}
%    %  \end{equation} 
%where $k(t)$ is a linearly-decreasing function in time $k(t) = \left(-\frac{1 + \frac{2}{C}LS}{T}\right) t + (1 + \frac{2}{C}LS)$, the TS algorithm will have a smaller regret in comparison to when $k(t) = 1 + \frac{2}{C}LS$.
%   
        
\section{Conclusion}\label{sec:conclusions}
We studied LB in which the environment is subject to unknown linear safety constraints that need to be satisfied at each round. For this problem, we proposed Safe-LTS, which to the best of our knowledge, is the first safe TS algorithm with provable regret guarantees. Importantly, we show that Safe-LTS achieves regret of the same order as in the original setting with no safety constraints. We have also compared Safe-LTS with several UCB-type safe algorithms showing that the former has: better regret in the worst-case, fewer parameters to tune and  often superior empirical performance. Interesting directions for future work include: theoretically studying the dynamic property of Section \ref{dyn.noise}, as well as, investigating TS-based alternatives to the GP-UCB-type algorithms of \cite{Krause,sui2018stagewise}.
% gaining a theoretical understanding of the regret of the algorithm when the TS distribution satisfies the dynamic property of Section \ref{dyn.noise}, which empirically leads to regret of smaller order.

\newpage
\bibliography{example_paper}
\bibliographystyle{icml2020}

\newpage
\appendix
\onecolumn

%%%%%%%%%%%%%%%%%%%%%%%%%%%%%%%%%%%%%%%%%%%%%%%%%%%%%%%%%%%%%%%%%%%%%%%%%%%%%%%
%%%%%%%%%%%%%%%%%%%%%%%%%%%%%%%%%%%%%%%%%%%%%%%%%%%%%%%%%%%%%%%%%%%%%%%%%%%%%%%

\section{Proof sketch: Why frequent optimism is enough to bound Term I}\label{sec:sketch_app}
As discussed in Section \ref{regret analysis}, the presence of the safety constraints complicates the requirement for optimism.  We show in Section \ref{sketch.of.the.proof} that Safe-LTS is optimistic with constant probability in spite of safety constraints. Based on this, we complete the sketch of the proof here by showing that we can bound the overall regret of Term I in \eqref{overall.regret} with the $V_{\tau}$-norm of optimistic (and in our case, safe) actions.   Let us first define the set of the optimistic parameters as
\begin{align}
    \Theta^{\rm{opt}}_t(\delta') = \{ \theta \in \mathbb{R}^d: \text{max}_{x \in \mathcal{D}_t^s} x^{\top} \theta \geq x_{\star}^{\top} \theta_{\star} \}. \label{optimisitc}
\end{align}
In Section \ref{sketch.of.the.proof}, we show that Safe-LTS samples from this set i.e., $\Tilde{\theta}_t \in \Theta_t^{\rm{opt}}$, with constant probability. Note that, if at round $t$ Safe-LTS samples from the set of  optimistic parameters, Term I at that round is non-positive. 
In the following, we show that  selecting the optimal arm corresponding to any optimistic parameter can control the overall regret of Term I. The argument below is adapted from \cite{abeille2017linear} with minimal changes; it is presented here for completeness.
 
 For the purpose of this proof sketch, we assume that at each round $t$, the safe decision set contains the previous safe action that the algorithm played, i.e., $ x_{t-1} \in \mathcal{D}_t^s$. However, for the formal proof in Appendix \ref{appndix.bounding.term I}, we do not need such an assumption.  
 Let $\tau$ be a time such that   $\Tilde{\theta}_{\tau} \in \Theta^{\rm{opt}}_t$, i.e., $ x_{\tau}^{\top} \Tilde{\theta}_{\tau} \geq x_{\star}^{\top} \theta_{\star}$. Then, for any $t \geq \tau$ we have \begin{align}
    \text{Term I} := R_t^{\text{TS}} =  x_{\star}^{\top} \theta_{\star} - & x_t^{\top} \Tilde{\theta}_t  \leq x_{\tau}^{\top} \Tilde{\theta}_{\tau} - x_t^{\top} \Tilde{\theta}_t \nonumber \\&\leq x_{\tau}^{\top} \left( \Tilde{\theta}_{\tau} - \Tilde{\theta}_t \right). \label{last.inequality} 
\end{align} The last inequality comes from the assumption that at each round $t$, the  safe decision set contains the previous played safe actions for rounds $s \leq t$; hence,  $x_{\tau}^{\top} \Tilde{\theta}_t \leq x_t^{\top} \Tilde{\theta}_t $. To continue from \eqref{last.inequality}, we use Cauchy-Schwarz, and obtain \begin{align}
     R_t^{\text{TS}} & \leq \norm{x_{\tau}}_{V_{\tau}^{-1}} \norm{\Tilde{\theta}_{\tau} - \Tilde{\theta}_t}_{V_{\tau}} \nonumber \\& \leq \left( \norm{\Tilde{\theta}_{\tau} - {\theta}_{\star}}_{V_{\tau}} + \norm{{\theta}_{\star} - \Tilde{\theta}_{t}  }_{V_{\tau}} \right) \norm{x_{\tau}}_{V_{\tau}^{-1}}. \nonumber \\& \leq \left( \norm{\Tilde{\theta}_{\tau} - {\theta}_{\star}}_{V_{\tau}} + \norm{{\theta}_{\star} - \Tilde{\theta}_{t}  }_{V_{t}} \right) \norm{x_{\tau}}_{V_{\tau}^{-1}} \label{cuachy-shc-bound}
\end{align} 
The last inequality comes from the fact that the Gram matrices construct a non-decreasing sequence ($V_{\tau}\preceq V_t $, $ \forall t \geq \tau$).  
Then, we define the ellipsoid $\mathcal{E}_t^{\text{TS}}(\delta')$ such that \begin{align}
    \mathcal{E}_t^{\text{TS}} (\delta') := \{ \theta \in \mathbb{R}^d : \norm{\theta - \hat{\theta}_t}_{V_t} \leq \gamma_t(\delta')    \} \label{e.TS},
\end{align} where \begin{align}
\gamma_t(\delta') = \beta_t(\delta') \big(  1 + \frac{2}{C}LS \big) \sqrt{c d \log{(\frac{c' d}{\delta})}}. \label{def.gamma}
\end{align}
 It is not hard to see by combining Theorem \ref{abbasi.2} and the concentration property that $\Tilde{\theta}_t \in \mathcal{E}_t^{\text{TS}}(\delta')$ with high probability. Hence, we can bound \eqref{cuachy-shc-bound} using triangular inequality such that: 
 \begin{align}
     R_t^{\text{TS}} &\leq \bigg( \gamma_{\tau}(\delta') + \beta_{\tau}(\delta') + \gamma_{t}(\delta') + \beta_{t}(\delta') \bigg) \norm{x_{\tau}}_{V_{\tau}^{-1}} \\& \leq  2 \bigg( \gamma_{T}(\delta') + \beta_{T}(\delta') \bigg) \norm{x_{\tau}}_{V_{\tau}^{-1}}
 \end{align}
The last inequality  comes from the fact that $\beta_t(\delta')$ and $\gamma_t(\delta')$ are non-decreasing in $t$ by construction.  Therefore, following the intuition of \cite{abeille2017linear}, we can upper bound  Term I with respect to the $V_{\tau}$-norm of the optimal safe action at time $\tau$ (see Section \ref{appndix.bounding.term I} in Appendix for formal proof). Bounding the term $\norm{x_{\tau}}_{V_{\tau}^{-1}}$ is standard based on the analysis provided in \cite{abbasi2011improved} (see Proposition \ref{proposition.bounding random actions} in the Appendix).

\section{Useful Results}\label{useful results}
The following result is standard and plays an important role in most  proofs for linear bandits problems.
\begin{Proposition}(\cite{abbasi2011improved})\label{proposition.bounding random actions}
Let $\lambda \geq 1$. For any arbitrary sequence of actions $(x_1,\dots,x_t) \in \mathcal{D}^t$, let $V_t$ be the corresponding Gram matrix \eqref{ Gram.matrix}, then \begin{align}
    \sum_{s=1}^{t} \norm{x_s}_{V_s^{-1}}^{2} \leq 2 \log{  \frac{\text{det}(V_{t+1})}{\text{det}(\lambda I)}} \leq 2 d \log{(1 + \frac{t L^2}{\lambda})}. \label{bound.random.actions}
\end{align}
In particular, we have \begin{align}
    \sum_{s=1}^{T} \norm{x_s}_{V_s^{-1}} \leq  \sqrt{T}\left(\sum_{s=1}^{T} \norm{x_s}_{V_s^{-1}}^{2}\right)^{\frac{1}{2}}  \leq  \sqrt{2 T d \log{ \left(1 + \frac{T L^2}{\lambda}\right)}}. \label{colorray.bound.over.arbitrary.sequance}
\end{align}
\end{Proposition}

Also, we recall the Azuma's concentration inequality for super-martingales.
\begin{Proposition}(Azuma's inequality \cite{boucheron2013concentration})
If a super-martingale $(Y_t)_{t \geq 0}$ corresponding to a filtration $\mathcal{F}_t$ satisfies $|Y_t - Y_{t-1} | < c_t$ for some positive constant $c_t$, for all $t=1,\dots,T$, then, for any $u > 0$,
\begin{align}
    \mathbb{P} \left(Y_T - Y_0 \geq u \right) \leq 2 e^{-\frac{u^2}{2 \sum_{t=1}^T c_t^2}}.
\end{align}
\end{Proposition}

\section{Confidence Regions}
\label{least-square confidence}
We start by constructing the following confidence regions for the RLS-estimates. 
\begin{definition}\label{def.events}
Let $\delta \in (0,1)$, $\delta' = \frac{\delta}{6T}$, and $t \in [T]$. We define the following events:
\begin{itemize}
    \item $\hat{E}_t$ is the event that the RLS-estimate $\hat{\theta}$ concentrates around $\theta_{\star}$ for all steps $s \leq t$, i.e., $\hat{E}_t = \{ \forall s \leq t, \norm{\hat{\theta}_s - \theta_{\star}}_{V_s} \leq \beta_s(\delta')  \}$;
    \item  $\hat{Z}_t$ is the event that the RLS-estimate $\hat{\mu}$ concentrates around $\mu_{\star}$, i.e., $\hat{Z}_t = \{ \forall s \leq t, \norm{\hat{\mu}_s - \mu_{\star}}_{V_s} \leq \beta_s(\delta')  \}$.
     Moreover, define  $Z_t$ such that $Z_t = \hat{E}_t \cap \hat{Z}_t$. 
     \item $\Tilde{E}_t$ is the event that the sampled parameter $\Tilde{\theta}_t$ concentrates around $\hat{\theta}_t$ for all steps $s \leq t$, i.e., $\Tilde{E}_t = \{ \forall s \leq t, \norm{\Tilde{\theta}_s - \hat{\theta}_s}_{V_s} \leq \gamma_s(\delta')  \}$. 
     Let $E_t$ be such that $E_t = \Tilde{E}_t \cap Z_t$.
      \end{itemize} 

\end{definition}

\begin{lemma}\label{first.event.bound}
    Under Assumptions \ref{assmp.1}, \ref{assmp.2}, we have $\mathbb{P}(Z) = \mathbb{P}(\hat{E} \cap \hat{Z}) \geq  1 - \frac{\delta}{3}$ where $\hat{E} = \hat{E}_T \subset \dots \subset \hat{E}_1$, and $\hat{Z} = \hat{Z}_T \subset \dots \subset \hat{Z}_1$.
\end{lemma}
\begin{proof}
The proof is similar to the one in Lemma 1 of \cite{abeille2017linear} and is ommited for brevity.
\end{proof}

\begin{lemma}\label{second.event.bound}
    Under Assumptions \ref{assmp.1}, \ref{assmp.2}, we have $ \mathbb{P}(E) = \mathbb{P}( \Tilde{E} \cap Z) \geq 1 - \frac{\delta}{2}$, where  $\Tilde{E} = \Tilde{E}_T \subset \dots \subset \Tilde{E}_1$. 
\end{lemma}
\begin{proof}
We  show that $\mathbb{P}(\Tilde{E}) \geq 1 - \frac{\delta}{6}$. Then, from Lemma \ref{first.event.bound} we know that  $\mathbb{P}(Z) \geq 1 - \frac{\delta}{3}$, thus we can conclude that $\mathbb{P}(E) \geq 1 - \frac{\delta}{2}$. Bounding $\Tilde{E}$ comes directly from concentration inequality \eqref{concen.}. Specifically, 
\begin{align}
     1 \leq   t \leq T, \quad &  \mathbb{P}\left( \norm{\Tilde{\theta}_t - \hat{\theta}_t}_{V_t}  \leq \gamma_t (\delta') \right)  = \mathbb{P}\left( \norm{\eta_t}_2 \leq \frac{\gamma_t(\delta')}{\beta_t(\delta')}  \right) \nonumber\\& = \mathbb{P} \left(  \norm{ \eta_t}_2 \leq \left( 1 + \frac{2}{C}LS \right) \sqrt{c d \log{(\frac{c' d}{\delta'})}} \right) \geq 1 - \delta'. \nonumber
\end{align}
Applying union bound on this ensures that $\mathbb{P}(\Tilde{E}) \geq 1 - T \delta' = 1 - \frac{\delta}{6}.$ 
 \end{proof}

 \section{Formal proof of Lemma \ref{optimitic.lemma}}\label{appndx.proof.of.lema}
 In this section, we provide a  formal statement and a detailed proof of Lemma \ref{optimitic.lemma}. Here, we need several modifications compared to \cite{abeille2017linear} that are required because in our setting, actions $x_t$ belong to inner approximations of the true safe set $\mathcal{D}_0^s$. Moreover, we follow an algebraic treatment that is perhaps simpler compared to the geometric viewpoint in \cite{abeille2017linear}.
 
 \begin{lemma}\label{optimitic.lemma.formal}
Let $\Theta_t^{\rm{opt}} = \{\theta \in \mathbb{R}^d : \rm{max}_{x \in \mathcal{D}_t^s}  x^{\top}\theta \geq x_{\star}^{\top}\theta_{\star}   \} \cap \mathcal{E}_t^{\text{TS}}$ be the set of optimistic parameters, $\Tilde{\theta}_t = \hat{\theta}_t + \beta_t(\delta') V_t^{-\frac{1}{2}} \eta_t$ with $\eta_t \sim \mathcal{D}^{\text{TS}}$, then $\forall t \geq 1$, $\mathbb{P}\left(\Tilde{\theta}_t \in \Theta_t^{\rm{opt}} | \mathcal{F}_t, Z_t \right) \geq \frac{p}{2}$.
\end{lemma}

\begin{proof}

 First, we provide the shrunk version $\Tilde{\mathcal{D}}_t^s$ of ${\mathcal{D}}_t^s$ as follows:
 
 \textbf{A shrunk safe decision set $\Tilde{\mathcal{D}}_t^s$.} Consider the enlarged confidence region $\Tilde{\mathcal{C}}_t$ centered at $\mu_{\star}$ as \begin{align}
    \Tilde{\mathcal{C}}_t := \{ v \in \mathbb{R}^d : \norm{v - \mu_{\star}}_{V_t} \leq 2 \beta_t(\delta')  \}.
\end{align}
We know that $\mathcal{C}_t \subseteq \Tilde{\mathcal{C}}_t $, since $\forall v \in \mathcal{C}_t $, we know that $\norm{v - \mu_{\star}}_{V_t} \leq \norm{v - \hat{\mu}_t}_{V_t} + \norm{\hat{\mu}_t - \mu_{\star}}_{V_t} \leq 2 \beta(t)$. From the definition of enlarged confidence region, we can get the following definition for shrunk safe decision set:
\begin{align}
    \Tilde{\mathcal{D}}_t^s  := & \{ x \in \mathcal{D}_0 : x^{\top} v \leq C , \forall v \in \Tilde{\mathcal{C}}_t \} = \{x \in \mathcal{D}_0 : \max_{v \in \Tilde{\mathcal{C}}_t} x^{\top} v \leq C  \} \nonumber \\& = \{x \in \mathcal{D}_0 : x^{\top} \mu_{\star} + 2 \beta_t(\delta') \norm{x}_{V_t^{-1}} \leq C  \}, \label{shrunk.safe.set}
\end{align}
and note that $\Tilde{\mathcal{D}}_t^s \subseteq {\mathcal{D}}_t^s$, and they are not empty, since they include zero due to Assumption \ref{assmp.3}.

Then, we define the parameter $\alpha_t$ such that the vector $z_t = \alpha_t x_{\star}$ in direction $x_{\star}$ belongs to $\Tilde{\mathcal{D}}_t^s$ and is closest to $x_{\star}$. Hence, we have: \begin{align}
    \alpha_t := \text{max} \left\{ \alpha \in [0,1] : z_t = \alpha x_{\star} \in \Tilde{\mathcal{D}}_t^s \right\}.
\end{align}
Since $\mathcal{D}_0$ is convex by Assumption \ref{assmp.3} and both  $0,x_{\star} \in \mathcal{D}_0$, we have 
\begin{align}
    \alpha_t = \text{max} \bigg\{ \alpha \in [0,1] : \alpha  \bigg( x^{\top}_{\star} \mu_{\star} +  2  \beta_t & (\delta' )  \norm{x_{\star}}_{V_t^{-1}}  \bigg)  \leq C \bigg\}. \label{def.alpha}
\end{align}
From constraint \eqref{safty.const}, we know that $x^{\top}_{\star} \mu_{\star} \leq C$. We choose $\alpha_t$ such that \begin{align}
    1 + \frac{2}{C} \beta_t(\delta') \norm{x_{\star}}_{V_t^{-1}} = \frac{1}{\alpha_t}. \label{alpha_t}
\end{align}
We need to study the probability that a sampled $\Tilde{\theta}_t$ drawn from $\mathcal{H}^{\text{TS}}$ distribution at round $t$ is optimistic, i.e., 
\begin{align*}
    p_t = \mathbb{P}\left( (x_t(\Tilde{\theta}_t))^{\top} \Tilde{\theta}_t \geq x_{\star}^{\top}\theta_{\star} \;\middle|\; \mathcal{F}_t, Z_t \right).
\end{align*}

Using the definition of $\alpha_t$ in \eqref{def.alpha}, we have
\begin{align}
   (x_t(\Tilde{\theta}_t))^{\top} \Tilde{\theta}_t = \max_{x \in \mathcal{D}_t ^s}  x^{\top}  \Tilde{\theta}_t   \geq \alpha_t x^{\top}_{\star} \Tilde{\theta}_t. \label{projection.ineq}
\end{align}
Hence, we can write
\begin{align}
    p_t  \geq & \mathbb{P} \left( \alpha_t x^{\top}_{\star} \Tilde{\theta}_t \geq x^{\top}_{\star} \theta_{\star}  \;\middle|\; \mathcal{F}_t, Z_t \right) \nonumber \\& = \mathbb{P} \left(  x^{\top}_{\star} \left(\hat{\theta}_t + \beta_t(\delta') V_t^{-\frac{1}{2}} \eta_t  \right) \geq \frac{x^{\top}_{\star} \theta_{\star}}{\alpha_t} \;\middle|\; \mathcal{F}_t, Z_t  \right) \nonumber
\end{align}
Then, we use the value that we chose for $\alpha_t$ in \eqref{alpha_t}, and we have
\begin{align*}
  =  \mathbb{P} \bigg(x^{\top}_{\star} \hat{\theta}_t +  \beta_t(\delta') x^{ \top}_{\star}  V_t^{-\frac{1}{2}} \eta_t  \geq x^{\top}_{\star} \theta_{\star} +  \frac{2}{C} \beta_t(\delta') \norm{x_{\star}}_{V_t^{-1}} x^{\top}_{\star} \theta_{\star} | \mathcal{F}_t, Z_t   \bigg) 
\end{align*}
 we know that $|x^{\top}_{\star} \theta_{\star}| \leq \|x_{\star}\|_2 \|\theta_{\star}\|_2 \leq LS$. Hence,
\begin{align*}
   p_t \geq  \mathbb{P}  \bigg( \beta_t(\delta') x^{\top}_{\star}  V_t^{-\frac{1}{2}}  \eta_t \geq  x^{\top}_{\star} ( \theta_{\star} - \hat{\theta}_t  )  + \frac{2}{C}LS \beta_t(\delta') \norm{x_{\star}}_{V_t^{-1}} \mid \mathcal{F}_t, Z_t  \bigg)
\end{align*}
From Cauchy–Schwarz inequality and \eqref{ellipsoid.RLS}, we have \begin{align*}
    \mid x^{\top}_{\star} \left( \theta_{\star} - \hat{\theta}_t  \right) \mid \leq  \norm{x_{\star}}_{V_t^{-1}} \norm{\theta_{\star} - \hat{\theta}_t}_{V_t}  \leq \beta_t(\delta') \norm{x_{\star}}_{V_t^{-1}}.
\end{align*}  
Therefore, we can write 
\begin{align}
    p_t \geq \mathbb{P} \bigg(  x^{ \top}_{\star} V_t^{-\frac{1}{2}}  \eta_t  \geq   \norm{x_{\star}}_{V_t^{-1}} + \frac{2}{C}LS  \norm{x_{\star}}_{V_t^{-1}} \mid \mathcal{F}_t, Z_t \bigg)\label{optimisem.inequality.1}
\end{align}
We define $u^{\top} = \frac{x^{\top}_{\star} V_t^{-\frac{1}{2}}}{\norm{x_{\star}}_{V_t^{-1}}}$, and hence $\norm{u}_2 = 1$. It follows from \eqref{optimisem.inequality.1} that
\begin{align}
    p_t \geq \mathbb{P} \left( u^{\top} \eta_t \geq 1 + \frac{2}{C}LS   \right) \geq p,
\end{align} where the last inequality follows the concentration inequality \eqref{concen.} of the TS distribution. We also need to show that the high probability concentration inequality event does not effect the TS of being optimistic. This is because the chosen confidence bound $\delta' = \frac{\delta}{6T}$ is small enough compared to the anti-concentration property \eqref{anti.concen}. Moreover, we assume that $T \geq \frac{1}{3p}$ which implies that $\delta' \leq \frac{p}{2}$. We know that for any events $A$ ans $B$, we have \begin{align}
    \mathbb{P}(A\cap B) = 1 - \mathbb{P}(A^c \cup B^c) \geq \mathbb{P}(A) - \mathbb{P}(B^c). \label{event.ineqaulity}
\end{align}
We  apply \eqref{event.ineqaulity} with $A = \{ J_t(\Tilde{\theta}_t) \geq J(\theta_{\star}) \} $ and $B = \{\Tilde{\theta}_t \in \mathcal{E}_t^{\text{TS}} \} $ that leads to \begin{align*}
    \mathbb{P}\left(\Tilde{\theta}_t \in \Theta_t^{\text{opt}} \;\middle|\; \mathcal{F}_t, Z_t \right) \geq p - \delta' \geq \frac{p}{2}.
\end{align*} 
\end{proof}

\section{Proof of Theorem \ref{Main.theorem}} \label{appdnx.proof.of.theorem}

The proof presented below follows closely the proof of \cite{abeille2017linear} and is primarily presented here for completeness.  Specifically, we have identified that the only critical change that needs to be made to account for safety is the proof of actions being frequently optimistic in the face of constraints thanks to the modified anti-concentration property \ref{anti.concen}. This was handled in  the previous section \ref{appndx.proof.of.lema}. For completeness, we also prove in Lemma \ref{first.action} that the first action of Safe-LTS is always safe under our assumptions.

We use the following decomposition for bounding the regret:

\begin{align}
  R(T)   \leq & \sum_{t=1}^T \left(x^{\top}_{\star} \theta_{\star} - x_t \theta_{\star}\right)  \mathbbm{1}\{E_t\} \nonumber\\& = \sum_{t=1}^{T} \left( \underbrace{ x_{\star}^{\top}\theta_{\star} - x_t^{\top}\Tilde{\theta}_t}_{\text{Term I}} \right)\mathbbm{1}\{E_t\}  +  \sum_{t=1}^T \left( \underbrace{ x_t^{\top}\Tilde{\theta}_t - x_t^{\top} \theta_{\star}}_{\text{Term II}} \right)\mathbbm{1}\{E_t\}.
\end{align}

\subsection{Bounding Term I.}\label{appndix.bounding.term I}

For any $\theta$, we denote  $ {x}_t({\theta}) = \text{arg}\max_{x \in \mathcal{D}_t ^s }  x^{\top} {\theta} $.
On the event $E_t$, $\Tilde{\theta}_t$ belongs to $\mathcal{E}_t^{\text{TS}}$ which leads to
\begin{align}
    (\text{Term I}) \mathbbm{1}\{E_t\}  :=  R_t^{\text{TS}} \mathbbm{1}\{E_t\} \leq  \left( x_{\star}^{\top}\theta_{\star} - \inf_{\theta \in \mathcal{E}_t^{\text{TS}}} (x_t(\theta))^{\top}\theta \right) \mathbbm{1}\{Z_t\}. \label{first.regret.bound}
\end{align}
Here and onwards, we use $\mathbbm{1}\{\mathcal{E}\}$ as the indicator function applied to an event $\mathcal{E}$. We have also used the fact that $E_t$ is a subset of $Z_t$.
Next, we can also bound  \eqref{first.regret.bound} by the expectation over any random choice of $\Tilde{\theta} \in \Theta_t^{\rm{opt}}$ (recall \eqref{optimisitc}) that leads to 
\begin{align}
    R_t^{\text{TS}} \leq \mathbb{E}\left[\left((x_t(\Tilde{\theta}))^{\top}\Tilde{\theta} - \inf_{\theta \in \mathcal{E}_t^{\text{TS}}} (x_t(\theta))^{\top}\theta\right) \mathbbm{1}\{Z_t\} \;\middle|\; \mathcal{F}_t, \Tilde{\theta} \in \Theta_t^{\text{opt}}  \right]. \notag
\end{align}
Equivalently,  we can write
\begin{align}
    R_t^{\text{TS}}  \leq \mathbb{E} \left[\sup_{\theta \in \mathcal{E}_t^{\text{TS}}} \left( \big({x}_t(\Tilde{\theta})\big)^{\top} \Tilde{\theta} - \big( x_t(\theta) \big)^{\top} \theta \right) \mathbbm{1}\{Z_t\} \;\middle|\; \mathcal{F}_t , \Tilde{\theta} \in \Theta_t^{opt}  \right], 
    \end{align}  
     
    Then, using Cauchy–Schwarz and the definition of $\gamma_t(\delta')$ in \eqref{def.gamma}
  \begin{align}
    & \mathbb{E} \left[\sup_{\theta \in \mathcal{E}_t^{\text{TS}}}  \left({x}_t (\Tilde{\theta}) \right)^{\top}  \left(\Tilde{\theta} - \theta \right) \mathbbm{1}\{Z_t\} \;\middle|\; \mathcal{F}_t , \Tilde{\theta} \in \Theta_t^{opt}  \right] \nonumber \\&\leq \mathbb{E} \left[ \norm{ x_t(\Tilde{\theta}) }_{V_t^{-1}}  \sup_{\theta \in \mathcal{E}_t^{\text{TS}}} \norm{ \Tilde{\theta} - \theta }_{V_t} \;\middle|\; \mathcal{F}_t , \Tilde{\theta} \in \Theta_t^{opt}, Z_t \right] \nonumber \mathbb{P}(Z_t) \nonumber\\& \leq 2 \gamma_t(\delta') \mathbb{E} \left[ \norm{ x_t(\Tilde{\theta}) }_{V_t^{-1}} \;\middle|\; \mathcal{F}_t , \Tilde{\theta} \in \Theta_t^{opt}, Z_t \right]  \mathbb{P}(Z_t). \nonumber
\end{align}

This property  shows that the regret $R_t^{\text{TS}}$ is upper bounded by $V_t^{-1}$-norm of the optimal safe action corresponding to the any optimistic parameter $\Tilde{\theta}$. Hence, we need to show that TS samples from the optimistic set with high frequency. We prove in Lemma \ref{optimitic.lemma} that TS is optimistic with a fixed probability ($\frac{p}{2}$) which leads to bounding  $R_t^{\text{TS}}$ as follows:
\begin{align}
   R_t^{\text{TS}} \frac{p}{2} & \leq 2 \gamma_t(\delta') \mathbb{E} \left[ \norm{ x_t(\Tilde{\theta}_t) }_{V_t^{-1}} \;\middle|\; \mathcal{F}_t , \Tilde{\theta}_t \in \Theta_t^{opt}, Z_t \right] \mathbb{P}(Z_t) \frac{p}{2}  \\& \leq  2 \gamma_t(\delta') \mathbb{E} \left[ \norm{  x_{t}(\Tilde{\theta}_t) }_{V_t^{-1}} \;\middle|\; \mathcal{F}_t , \Tilde{\theta}_t \in \Theta_t^{opt}, Z_t \right] \mathbb{P}(Z_t) \mathbb{P} \left(\Tilde{\theta}_t \in \Theta_t^{\text{opt}} \;\middle|\; \mathcal{F}_t,Z_t \right) \nonumber \\& \leq  2 \gamma_t(\delta') \mathbb{E} \left[ \norm{ x_{t}(\Tilde{\theta}_t) }_{V_t^{-1}} \;\middle|\; \mathcal{F}_t ,  Z_t \right] \mathbb{P}(Z_t). 
\end{align}
By reintegrating over the event $Z_t$ we get \begin{align}
   R_t ^{\text{TS}} \leq  \frac{4 \gamma_t(\delta')}{p} \mathbb{E} \left[ \norm{ x_{t}(\Tilde{\theta}_t) }_{V_t^{-1}} \mathbbm{1}\{Z_t\} \;\middle|\; \mathcal{F}_t \right]. \label{regret.bound.over.Rts.}
\end{align}
Recall that $E_t \subset Z_t$, hence \begin{align}
    R^{\text{TS}}(T) \leq  \sum_{t=1}^T R_t^{\text{TS}} \mathbbm{1}\{E_t\} \leq \nonumber  \frac{4 \gamma_T(\delta')}{p}  \sum_{t=1}^T \mathbb{E} \left[ \norm{ x_{t}(\Tilde{\theta}_t) }_{V_t^{-1}} \;\middle|\; \mathcal{F}_t \right].
\end{align}
 
 For bounding this term, we rewrite the RHS above as:
 \begin{align}
      R^{\text{TS}}(T)  \leq  \bigg(  \sum_{t=1}^T \norm{x_t}_{V_t^{-1}}  +  \sum_{t=1}^T  \left( \mathbb{E} \left[ \norm{ x_{t}(\Tilde{\theta}_t) }_{V_t^{-1}} \;\middle|\; \mathcal{F}_t \right] - \norm{x_t}_{V_t^{-1}} \right)      \bigg).
 \end{align}
We can now bound the first expression using Proposition \ref{proposition.bounding random actions}. For the second expression we proceed as follows:
 \begin{itemize}
     \item First, the sequence $$Y_t = \sum_{s=1}^t  \left( \mathbb{E} \left[ \norm{ x_{s}(\Tilde{\theta}_s) }_{V_s^{-1}} \;\middle|\; \mathcal{F}_s \right] - \norm{x_s}_{V_s^{-1}} \right)$$ is a martingale by construction. 
    \item Second, under Assumption \ref{assmp.3}, $\norm{x_t}_2 \leq L$, and since $V_t^{-1} \leq \frac{1}{\lambda} I$, we  can write \begin{align}
     \mathbb{E} \left[ \norm{ x_{s}(\Tilde{\theta}_s) }_{V_s^{-1}} \;\middle|\; \mathcal{F}_s \right] - \norm{x_s}_{V_s^{-1}} \leq \frac{2 L}{\sqrt{\lambda}}, \forall t \geq 1.
 \end{align}
  \item Third, for bounding  $Y_T$, we use Azuma's inequality, and we have that with probability $1- \frac{\delta}{2}$, \begin{align}
      Y_T \leq \sqrt{\frac{8T L^2}{\lambda} \log{\frac{4}{\delta}}}.
  \end{align} \end{itemize}
  Putting these together, we conclude that  with probability $1-\frac{\delta}{2}$,
  \begin{align}
      R^{\text{TS}}(T) \leq  \frac{4 \gamma_T(\delta')}{p} \bigg( \sqrt{2 T d \log{ \big(1 +  \frac{T L^2}{\lambda}\big)}} +   \sqrt{\frac{8T L^2}{\lambda} \log{\frac{4}{\delta}}} \bigg).
  \end{align}
  
  \subsection{Bounding Term II}\label{appndix.bound.term II}
  We can bound on Term II using the general result of \cite{abbasi2011improved}. In fact, we can use the following general decomposition:
  \begin{align}
     & \sum_{t=1}^T (\text{Term II}) \mathbbm{1}\{E_t\}  := R^{\text{RLS}}(T)  \nonumber\\& =  \sum_{t=1}^T \left( x_t^{\top} \Tilde{\theta}_t - x_t^{\top} \theta_{\star}\right) \mathbbm{1}\{E_t\}  \leq \sum_{t=1}^T \mid x_t^{\top} (\Tilde{\theta}_t - \hat{\theta}_t) \mid \mathbbm{1}\{E_t\} +   \sum_{t=1}^T \mid x_t^{\top} (\hat{\theta}_t - \theta_{\star}) \mid \mathbbm{1}\{E_t\}. 
  \end{align} By Definition \ref{def.events}, we have $E_t \subseteq Z_t$ and $E_t \subseteq \Tilde{E}_t$, and hence \begin{align*}
      & \mid x_t^{\top} (\Tilde{\theta}_t - \hat{\theta}_t) \mid \mathbbm{1}\{E_t\} \leq \norm{x}_{V_t^{-1}} \gamma_t(\delta') \\& 
      \mid x_t^{\top} (\hat{\theta}_t - \theta_{\star}) \mid \mathbbm{1}\{E_t\} \leq \norm{x}_{V_t^{-1}} \beta_t(\delta'). 
  \end{align*} Therefore, from Proposition \ref{proposition.bounding random actions}, we have with probability $ 1- \frac{\delta}{2}$ \begin{align}
      R^{\text{RLS}} (T) \leq \left( \beta_T(\delta') + \gamma_T (\delta') \right) \sqrt{2 T d \log{ \left(1 + \frac{T L^2}{\lambda}\right)}}.
  \end{align}
  
  \subsection{Overall Regret Bound}
  Recall that from \eqref{overall.regret}, $R(T)\leq R^{TS}(T) +R^{\text{RLS}} (T)$. As shown previously,  each term is bounded separately with probability $1-\frac{\delta}{2}$. Using union bound over two terms, we get the following expression: 
  \begin{align}
      R(T) \leq  ( \beta_T(\delta') + \gamma_T(\delta') (1+\frac{4}{p}) ) \sqrt{2 T d \log{ (1 + \frac{T L^2}{\lambda})}}  + \frac{4 \gamma_T(\delta')}{p} \sqrt{\frac{8T L^2}{\lambda} \log{\frac{4}{\delta}}},
  \end{align}  holds with probability $1-\delta$ where $\delta' = \frac{\delta}{6T}$.

For completeness we show below that action $x_1$ is safe. Having established that, it follows that the rest of the actions $x_t, t> 1$ are also safe with probability at least $1-\delta'$. This is by construction of the feasible sets $\mathcal{D}_t^s$ and by the fact that $\mu_\star\in\mathcal{C}_t(\delta')$ with the same probability for each $t$.

\begin{lemma} \label{first.action}
    The first action that Safe-LTS chooses is safe, that is $x_1^{\top} \mu_{\star} \leq C$.
\end{lemma}
\begin{proof}
At round $t=1$, the RLS-estimate $\hat{\mu}_1 = 0$ and $V_1 = \lambda I$. Thus, Safe-LTS chooses the action which maximizes the expected reward while satisfying $x_1^{\top} \hat{\mu}_1 + \beta_1(\delta') \norm{x_1}_{V_1^{-1}} \leq C$. Hence, $x_1$ satisfies:
 \begin{align*}
     \beta_1(\delta') \norm{x_1}_{V_1^{-1}} \leq C. \label{fist.action}
 \end{align*}
From \eqref{eq:beta_t} and $V_1^{-1} = (1/\lambda) I$ leads to $ S \norm{x_1}_2 \leq C$ which completes the proof. 
\end{proof}

\section{More on experimental results}\label{simulationssss}
\paragraph{Standard deviations.} Figure \ref{fig:secondcomparison.variance} shows the sample standard deviation of regret around the average per-step regret for each one of the curves depicted in Figure \ref{fig:regretcomparison}. We remark on the strong dependency of the performance of  LUCB-based algorithms on the specific problem instance, whereas the performance of Safe-LTS does not vary significantly under the same instances.

  \begin{figure*}[t!]
% \centering
  \includegraphics[width=0.325\linewidth]{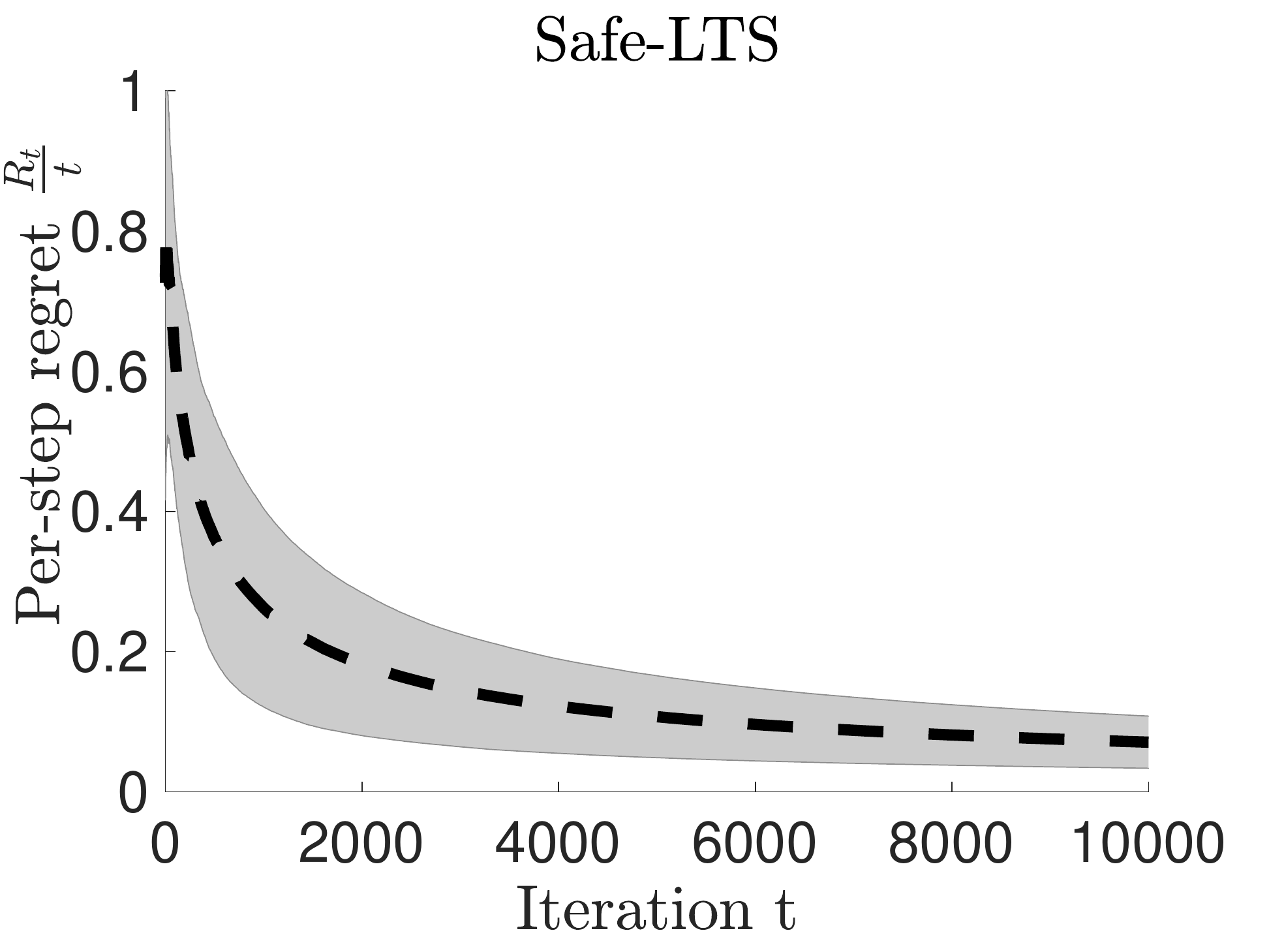}
    % \caption{Growth of safe sets in Safe-LUCB}   
   \includegraphics[width=0.325\linewidth]{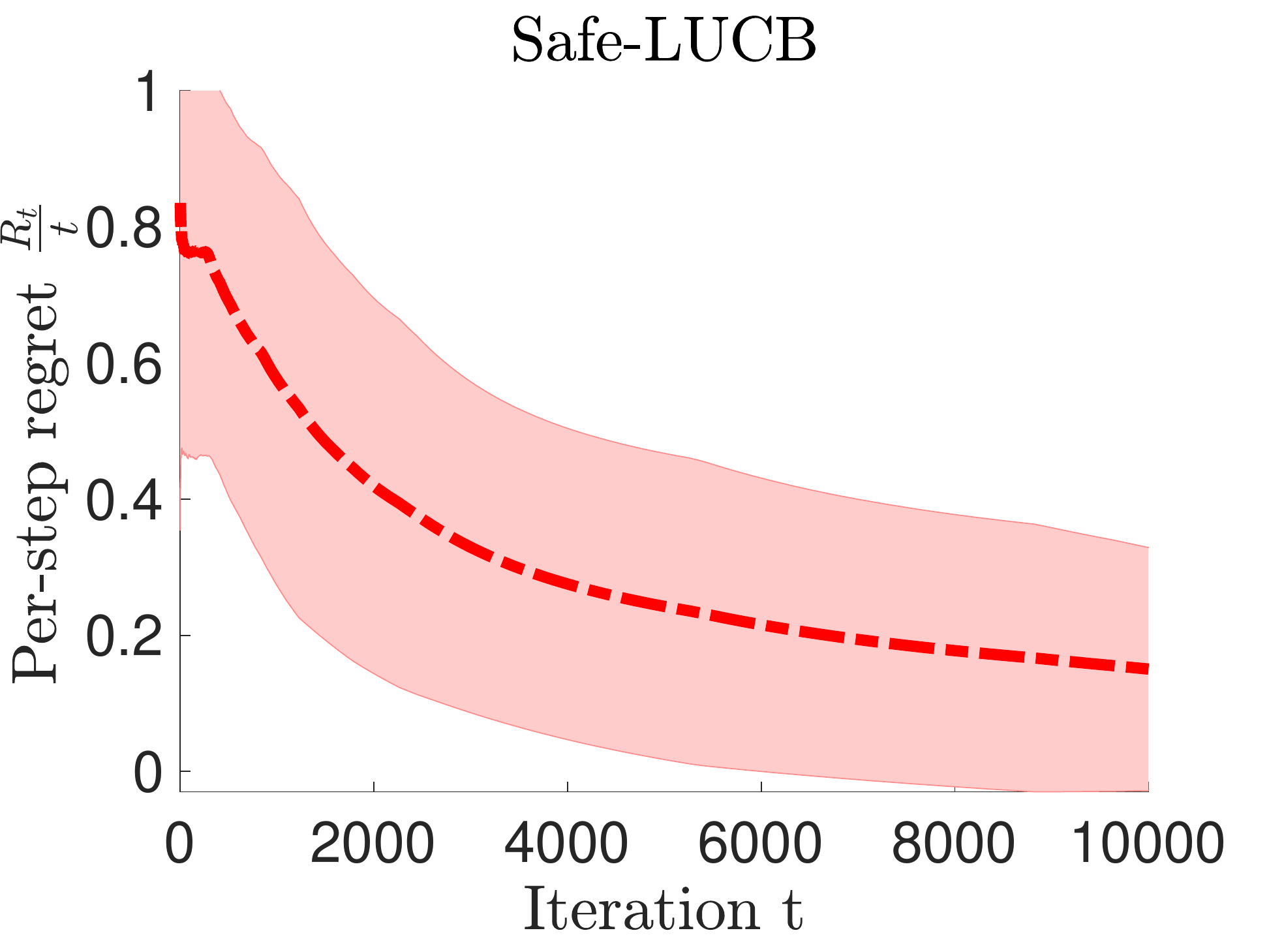}
    %   \caption{Growth of safe sets in Safe-LTS}
\includegraphics[width=0.325\linewidth]{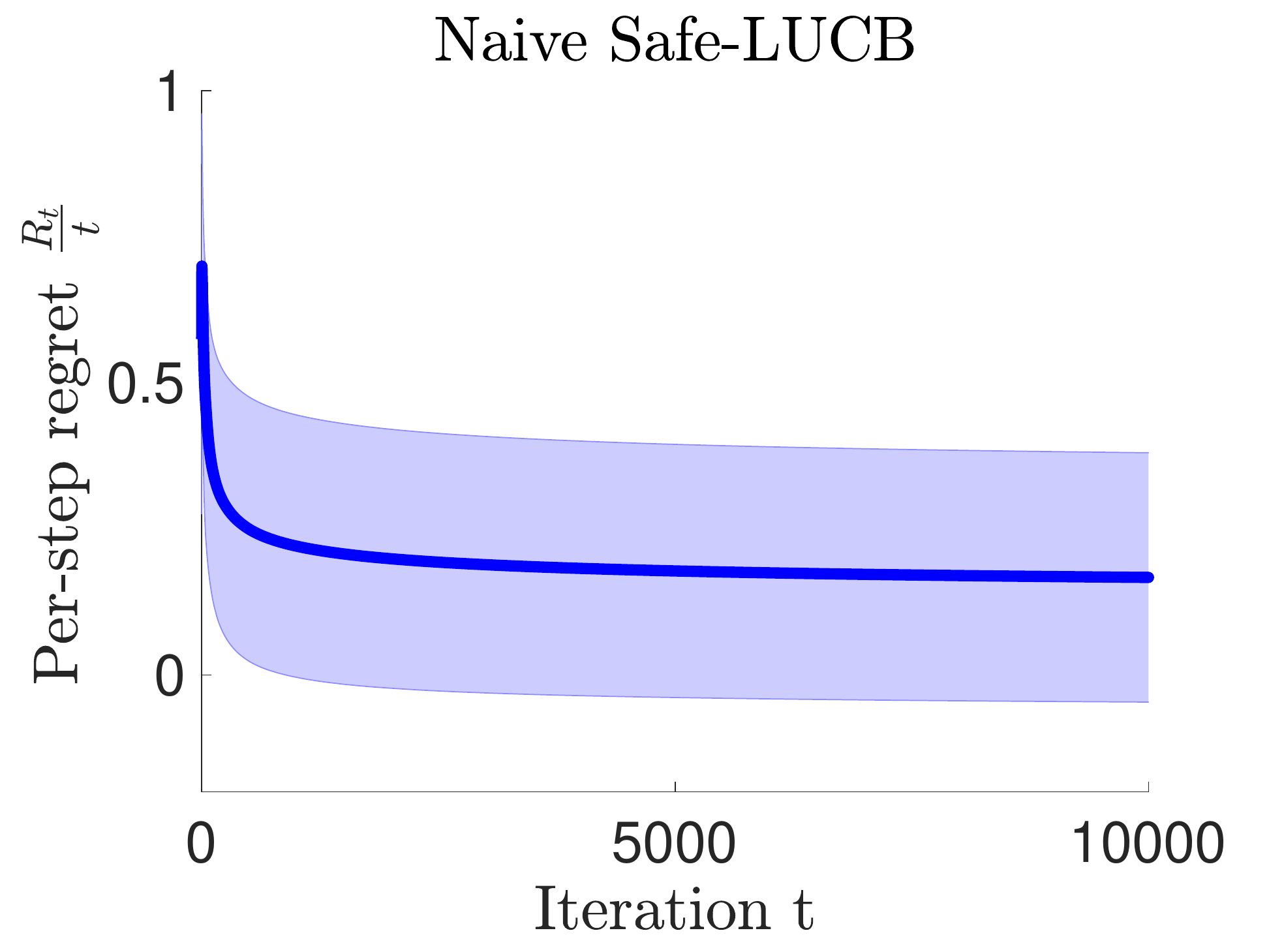}
        \caption{Comparison of mean per-step regret for Safe-LTS, Safe-LUCB, and Naive Safe-LUCB. The shaded regions show one standard deviation around the mean. The results are averages over 30 problem realizations.}
  \label{fig:secondcomparison.variance}
 \end{figure*}
 
 \paragraph{On the ``Inflated Safe-LUCB".} Here,  we  provide further evidence on the performance of the ``Inflated Naive Safe-LUCB" that we introduced in Section \ref{comparison with safe-ucb}. Recall from Figure \ref{inflated_regret} that ``Inflated Safe-LUCB" does \emph{not} always lead to proper safe set expansion and hence good regret performance. Specifically, the regret curves in Figure \ref{inflated_regret} are for a problem instance with  $\theta_* = \begin{bmatrix} 0.5766\\-0.1899 \end{bmatrix}$, $\mu_* = \begin{bmatrix} 0.2138\\-0.0020 \end{bmatrix}$, and $C = 0.0615 $. 
 %in order to show that both of the above algorithms may Fig. \ref{inflated_regret} illustrates this by showing the average regret in the above scenario, where both algorithms show nearly linear regret (for more detailed illustrations, see Section \ref{simulationssss} of the Appendix).
 %we  showed a problem instance in which simply inflating the confidence set of Naive Safe-LUCB by $1+\frac 2 C LS$ to favor larger exploration does \emph{not} lead to proper safe set expansion and hence good regret performance. The 
 While ``Inflated Naive Safe-LUCB" does not achieve the logarithmic regret of Safe-LTS in general, our numerical simulations suggest that the former actually outperforms the latter in 
 %However, we also would like to highlight that through the same modification, Inflated Naive Safe-LUCB will outperform Safe-LTS 
 %in most 
 randomly generated instances. This is illustrated in  Fig.\ref{fig:better.inflated}. %However, as previously noted, this appears to not be the case in general (see Fig. \ref{inflated_regret}).

  \begin{figure}
     \centering
          \includegraphics[width=0.4\textwidth]{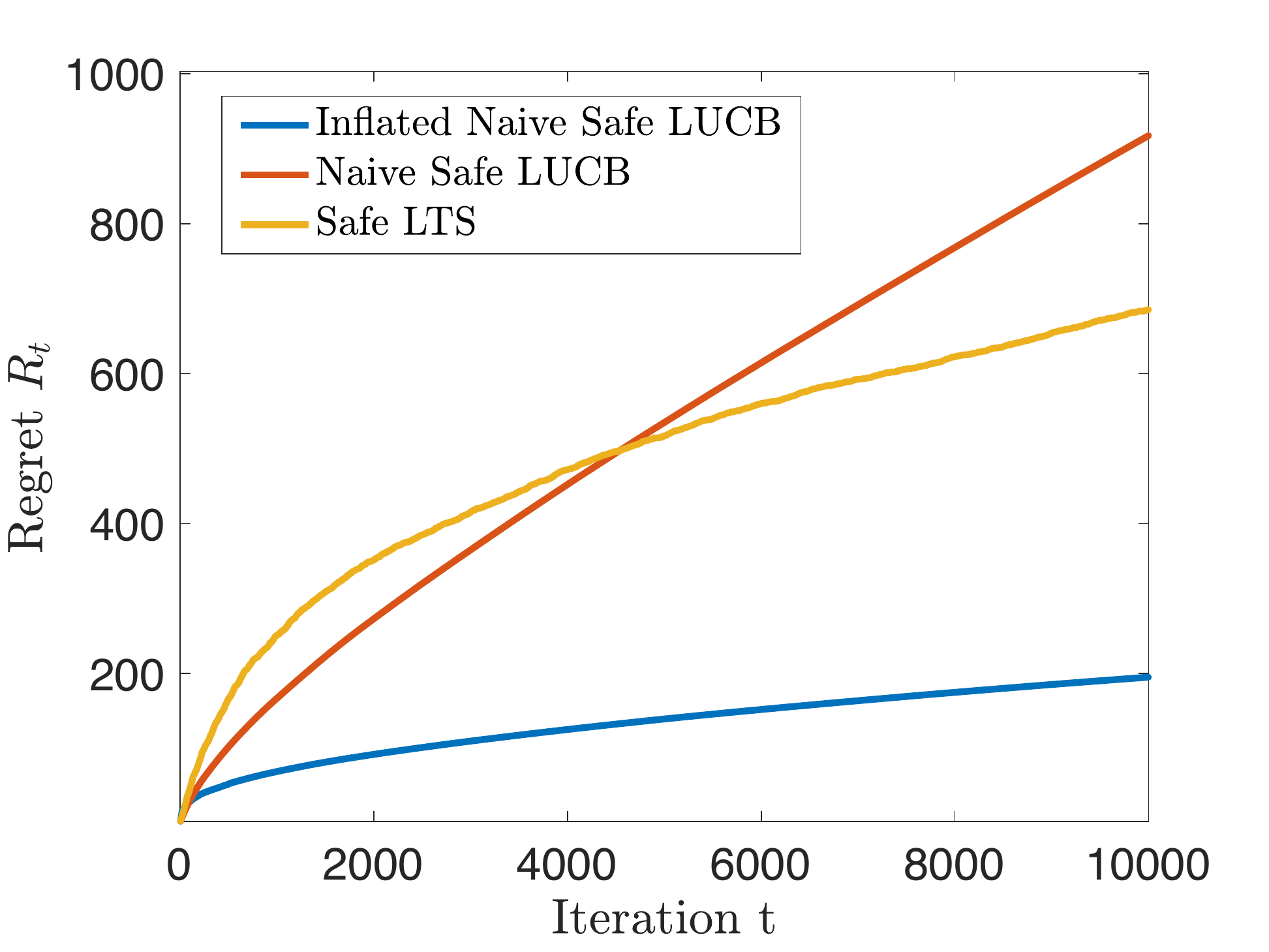}
         \caption{Comparison of the cumulative regret of   Safe-LTS and Naive Safe-LUCB and Inflated Naive Safe-LUCB algorithms over randomly generated instances.}
         \label{fig:better.inflated}
   \end{figure}

\paragraph{On safe-set expansion.} Fig. \ref{fig:safesets.expansion} highlights the gradual expansion of the safe decision set for Safe-LUCB in \cite{amani2019linear} and Safe-LTS for a problem instance  in which the safety constraint is active for parameters $\theta_* = \begin{bmatrix}
    0.9\\0.23 \end{bmatrix}$, $\mu_* = \begin{bmatrix}
    0.55\\-0.31 \end{bmatrix}$, and $C = 0.11 $... Similarly, Fig. \ref{fig:compar.inflated.lts} illustrates the expansion of the safe decision set for ``Inflated Naive Safe-LUCB" and Safe-LTS  for a problem instance with parameters $\theta_* = \begin{bmatrix}
    0.5766\\-0.1899 \end{bmatrix}$, $\mu_* = \begin{bmatrix}
    0.2138\\-0.0020 \end{bmatrix}$, and $C = 0.0615 $ in which the former provides poor (almost linear) regret. These empirical experiments reinforce the main message of our paper that the inherent randomized nature of TS is crucial for properly expanding the safe action set.

% Regarding the performance of LTS with dynamic noise distribution in Fig \ref{regret}, we need to  note that  as we discussed intuitively in Section \ref{challenges of safety},  at the first rounds of algorithm, we do not have an accurate estimate of  $\theta_{\star}$. Hence, in order to favor optimism \eqref{eq:optimism} in the first stages of the algorithm , we need to sample more aggressively, i.e., $\tilde{\theta}_t$ should have a larger norm than $\theta_{\star}$. As the algorithm progresses and we collect more actions that result in a better estimate of $\theta_{\star}$, explorations do not need to be so aggressive, otherwise it may cause  the growth of the regret.  The linear-decreasing scheme is the first natural attempt in quantifying this.

\begin{figure*}
\centering
  \includegraphics[width=0.45\linewidth]{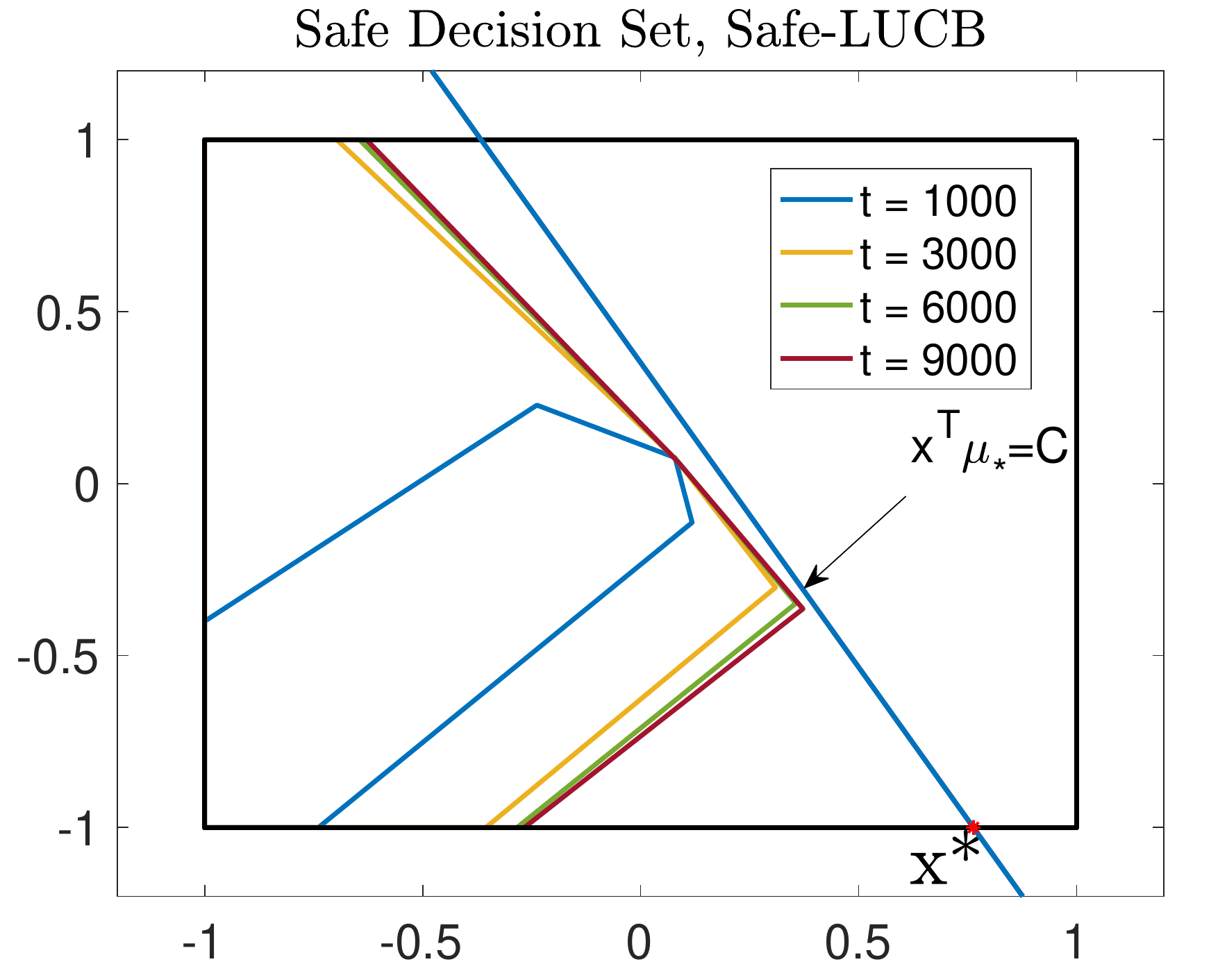}
     % \caption{Growth of safe sets in Safe-LUCB}   
  \includegraphics[width=0.45\linewidth]{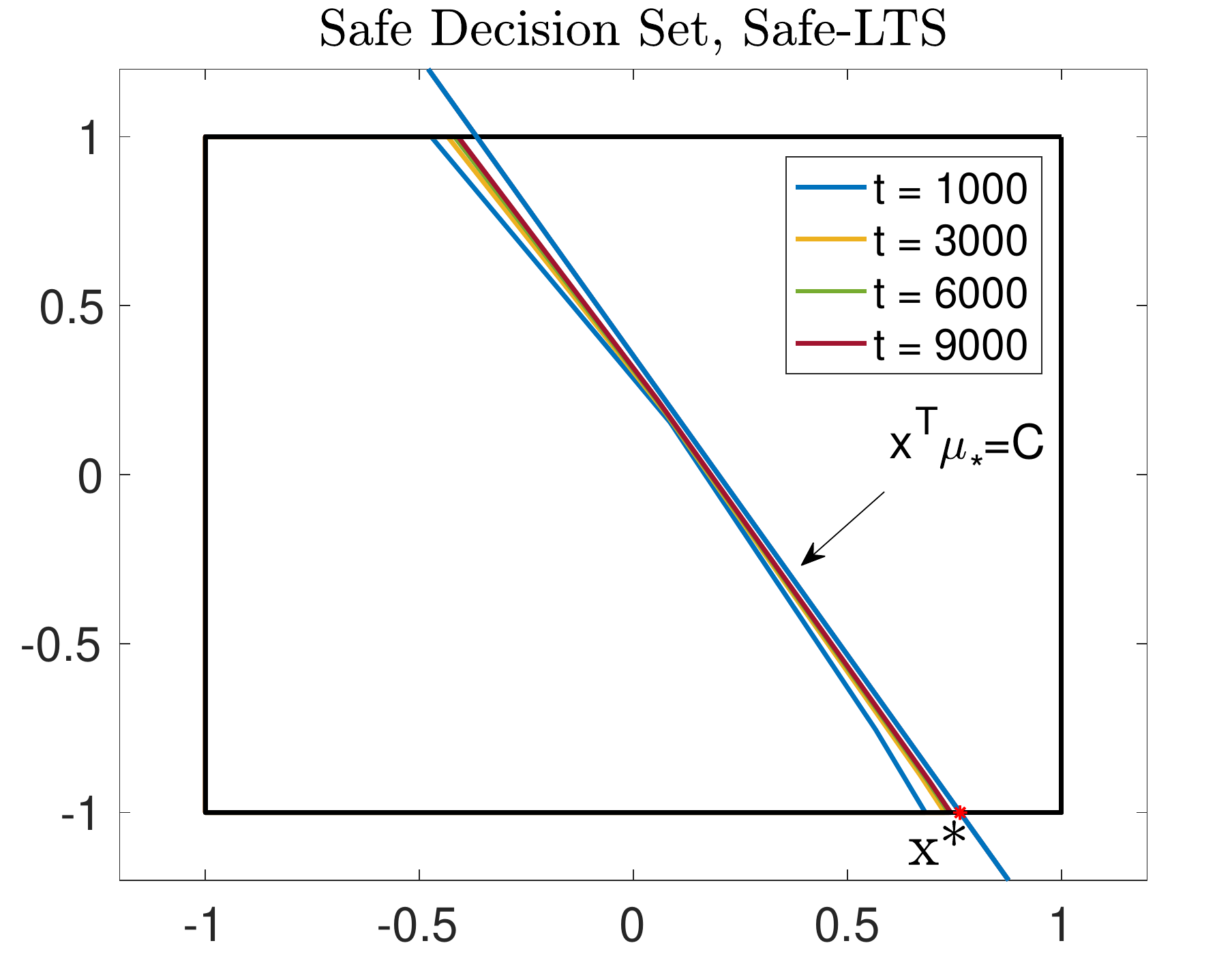}
%     %   \caption{Growth of safe sets in Safe-LTS}
% \includegraphics[width=0.8\linewidth]{regret.eps}
        \caption{ Comparison of  expansion of a safe decision sets for  Safe-LUCB and Safe-LTS, for a single problem instance.  }
  \label{fig:safesets.expansion}
 \end{figure*}

   \begin{figure*}[t!]
% \centering
  \includegraphics[width=0.45\linewidth]{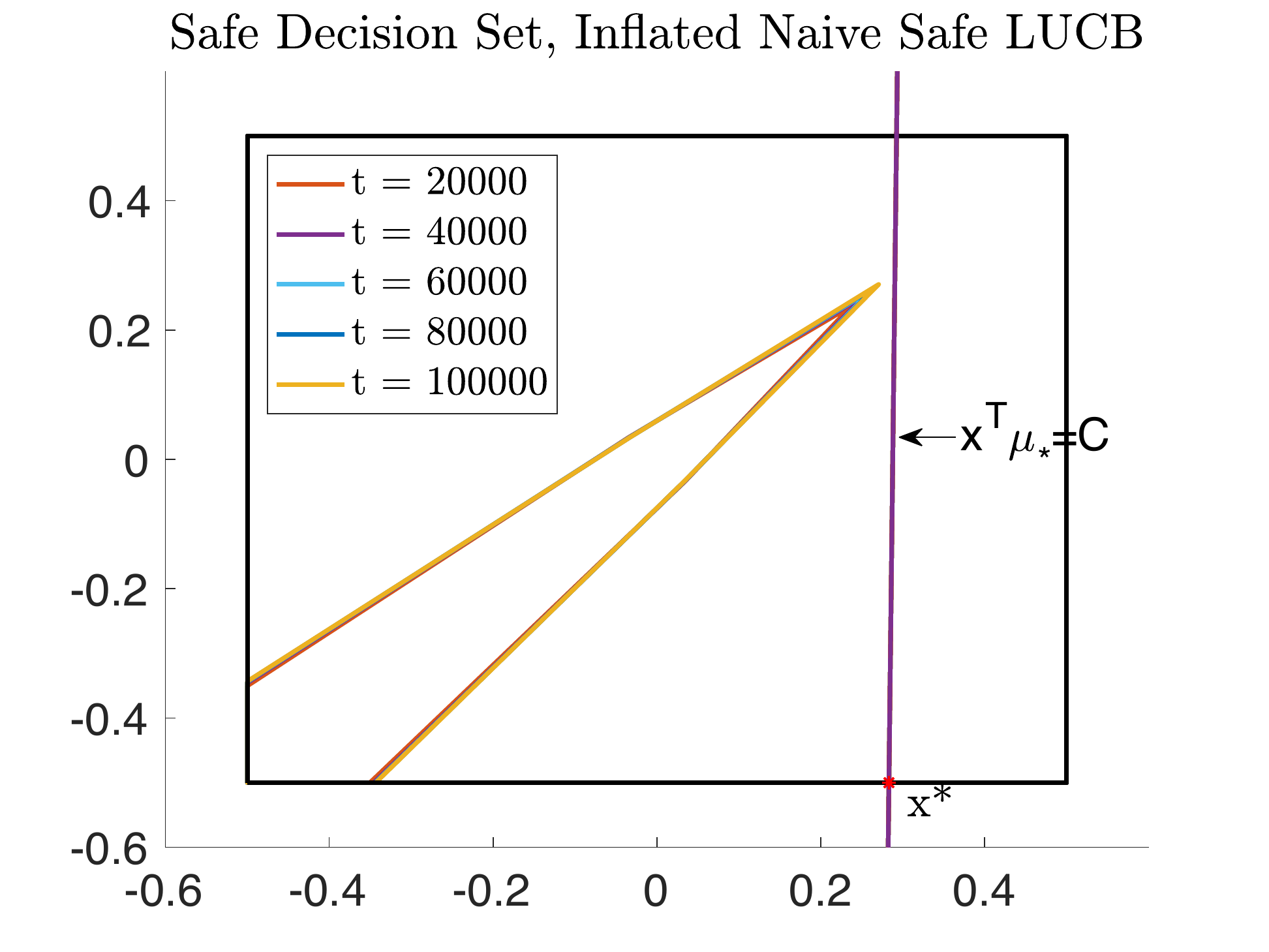}
    % \caption{Growth of safe sets in Safe-LUCB}   
   \includegraphics[width=0.39\linewidth]{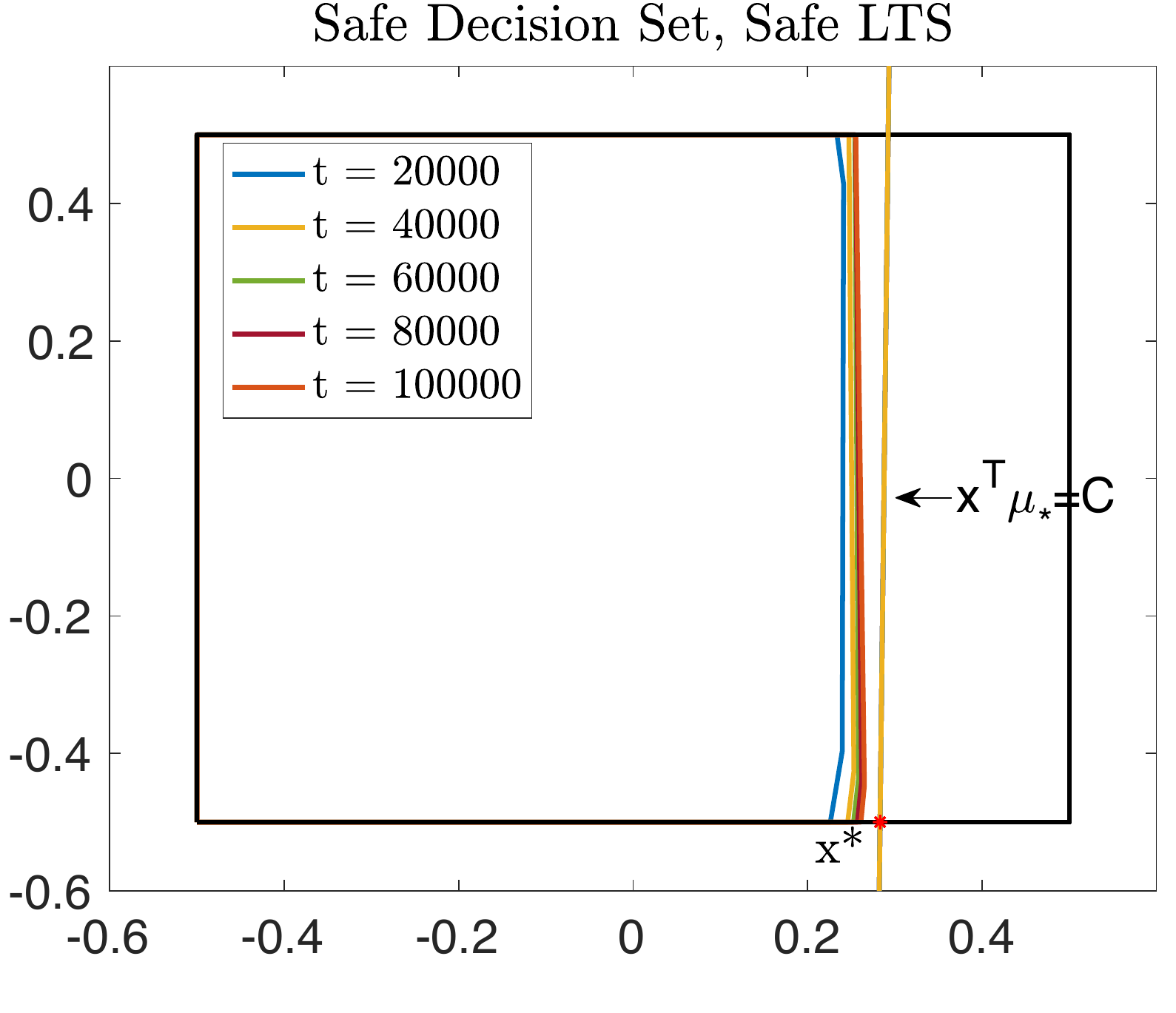}
    %   \caption{Growth of safe sets in Safe-LTS}
        \caption{Comparison of expansion of safe decision sets for Safe-LTS,  and Inflated Naive Safe-LUCB.}
  \label{fig:compar.inflated.lts}
 \end{figure*}

\end{document}